\documentclass[a4paper]{elsarticle}
\usepackage{amsmath}
\usepackage{amsthm}
\usepackage{cases}
\usepackage{threeparttable}
\usepackage{algorithm} 
\usepackage{algorithmic}
\usepackage{color}

\emergencystretch=\maxdimen
\newtheorem{theorem}{Theorem}
\newtheorem{lemma}{Lemma}
\newtheorem{proposition}{Proposition}

\journal{XXXXX}

\begin{document}

\begin{frontmatter}

\title{Sparse Algorithm 
for Robust LSSVM in Primal Space}


\author[mymainaddress,mysecondaryaddress]{Li Chen}
\ead{lilichenhappy@163.com}
\author[mymainaddress]{Shuisheng Zhou\corref{mycorrespondingauthor}}
\cortext[mycorrespondingauthor]{Corresponding author}
\ead{sszhou@mail.xidian.edu.cn}

\address[mymainaddress]{School of Mathematics and Statistics, Xidian University, 266 Xinglong Section, Xifeng Road, Xi'an, China}
\address[mysecondaryaddress]{Department of Basic Science, College of Information and Business, Zhongyuan Technology University, 41 Zhongyuan Middle Road, Zhengzhou, China}
\begin{abstract}
As enjoying the closed form solution, least squares support vector machine (LSSVM) has been widely used for classification and regression problems having the comparable performance with other types of SVMs. However, LSSVM has two drawbacks: sensitive to outliers and lacking sparseness. Robust LSSVM (R-LSSVM) overcomes the first partly via nonconvex truncated loss function, but the current algorithms for R-LSSVM with the dense solution are faced with the second drawback and are inefficient for training large-scale problems. In this paper, we interpret the robustness of R-LSSVM from a re-weighted viewpoint and give a primal R-LSSVM by the representer theorem. The new model may have sparse solution if the corresponding kernel matrix has low rank. Then approximating the kernel matrix by a low-rank matrix and smoothing the loss function by entropy penalty function, we propose a convergent sparse R-LSSVM (SR-LSSVM) algorithm to achieve the sparse solution of primal R-LSSVM, which overcomes two drawbacks of LSSVM simultaneously. The proposed algorithm has lower complexity than the existing algorithms and is very efficient for training large-scale problems. Many experimental results illustrate that SR-LSSVM can achieve better or comparable performance with less training time than related algorithms, especially for training large scale problems.
\end{abstract}


\end{frontmatter}


\section{Introduction} \label{intro}
Least squares support vector machine (LSSVM) was introduced by Suykens\cite{Suykens1999} and has been a powerful learning technique for classification and regression. It has been successfully used in many real world pattern recognition problems, such as disease diagnosis\cite{Duygu2011}, fault detection\cite{Long2014}, image classification \cite{Yang2015}, partial differential equations solving\cite{Mehrkanoon2015} and  visual tracking\cite{Gao2016}. LSSVM tries to minimize least squares errors on the training samples. Comparing with other SVMs, LSSVM is based on equality constraints rather than inequality ones, hence it has closed form solutions by solving a system of linear equations instead of solving a quadratic programming (QP) problem iteratively as other SVMs. So the training of LSSVM is simpler than other SVMs.

However, LSSVM has two main drawbacks. One is that it is sensitive to outliers, because outliers always have large support values (the values of Lagrange multiplier), which means that the influences of outliers are larger than other samples in constructing the decision function. Another is that the solution of LSSVM lacks sparse, which limits the method for training large scale problems.

In order to overcome the sensitivity to outliers of the LSSVM, Suykens et al.\cite{Suykens2002} proposed the weighted LSSVM (W-LSSVM) model by putting small weights on the less important samples or outliers to reduce their influence to the model. Some other weight setting strategies are proposed, see \cite{Valyon2003}\cite{You2011}. Theoretical analyses and the experimental results indicate that such methods are robust to outliers. But those methods need pre-solve the original LSSVM to set the weights, so they are all not suit for training large scale problems. Another technique to deal with robustness is on non-convex loss functions. Based on truncated least squares loss function, Wang et al.\cite{KuainiWang2014} and Yang et al.\cite{XiaoweiYang2014} presented robust LSSVM (R-LSSVM) model. Experimental results show that R-LSSVM model significantly reduces the effect of the outliers. However, the solutions to R-LSSVM by Yang's and Wang's algorithms both lack sparseness, and they need pre-compute the whole kernel matrix $K$ and the inverse of $(\lambda I+K)$, hence they are both time consuming for the large scale data sets. They are even unable to handle the data sets containing more than 10,000 training samples on common computers.

There are also some methods to promote the sparsity of LSSVM. Suykens et al.\cite{Suykens2000}\cite{J.A.K.Suykens2002} proposed a pruning algorithm which iteratively remove a small amount of samples (5\%) with smallest support values to impose sparseness. In this pruning algorithm, a retraining of LSSVM with the reduced training set is needed for each iteration, which leads to a large computation cost. Fixed-size least squares support vector machine (FS-LSSVM)\cite{Suykens2002} is another sparse algorithm. In this algorithm, some support vectors (SVs) referred to as prototype vectors are fixed in advance, and then they are replaced iteratively by samples which are randomly selected from the training set based on the quadratic R\'{e}nyi entropy criterion. However, in each iteration, this method only computes the entropy of the samples that are selected in the working set rather than the whole data set, which may cause the sub-optimized solutions. Jiao et al.\cite{Jiao2007} presented the fast sparse approximation for LSSVM (FSA-LSSVM), in which an approximated decision function was built iteratively by adding the basis function from a kernel-based dictionary one by one until the $\varepsilon$ criterion satisfied. This algorithm obtains sparse classifiers at a rather low cost. But with the very sparse setting, the experimental results in \cite{sszhou2016} show that FSA-LSSVM is not good on some training data sets. Zhou\cite{sszhou2016} proposed pivoting Cholesky of primal LSSVM (PCP-LSSVM) which is an iterative method based on incomplete pivoting Cholesky factorization of the kernel matrix. Theoretical analyses and the experimental results indicate that PCP-LSSVM can obtain acceptable test accuracy by extreme sparse solution.


In this paper, we aim to obtain the sparse solution of the R-LSSVM model to overcome the two drawbacks of LSSVM simultaneously. New algorithm solves the R-LSSVM in primal space as Zhou\cite{sszhou2016} did for LSSVM, and our main contributions can be summarized as follows:
\begin{itemize}%
\item By introducing an equivalent form of the truncated least squared loss function, we show that R-LSSVM is equivalent to a re-weighted LSSVM model, which explains the robustness of R-LSSVM.
\item We illustrate that representer theorem is also held for the non-convex loss function, and propose the primal R-LSSVM model which has a sparse solution if the kernel matrix is low rank.
\item We propose sparse R-LSSVM algorithm to obtain the sparse solution of R-LSSVM by applying low-rank approximation of the kernel matrix. The complexity of the new algorithm is lower than the existing non-sparse R-LSSVM algorithms.
\item A large number of experiments demonstrate that the proposed algorithm can process large-scale problems efficiently.
\end{itemize}

The rest of the paper is organized as follows. The brief descriptions of the R-LSSVM and its existing algorithms are given in section 2. In section 3, robustness of R-LSSVM is interpreted from a re-weighted viewpoint. In section 4, primal R-LSSVM and its smooth version are discussed, and the novel sparse algorithm is proposed. After that, the convergence and complexity of the new algorithm are analyzed. Section 5 includes some experiments to show the efficiency of the proposed algorithm. Section 6 concludes this paper.

\section{Robust LSSVM model and the existing algorithms}
\label{sec:1}
In this section, we briefly summarize the R-LSSVM and the existing algorithms.
\subsection{Robust LSSVM}

Consider a training set with $m$ pairs of samples $\left\{ {{\mathbf x_i},{y_i}} \right\}_{i = 1}^m$, where ${\mathbf x_i} \in {\Re^l}$ are the input data and ${y_i} \in \left\{ { - 1, + 1} \right\}$ or ${y_i} \in \Re$ are the output targets corresponding to the inputs for classification or regression problems. The classical LSSVM model is described as follows:
\begin{equation}\label{eq:LSSVM}
\min_{\mathbf w\in\Re^m,b\in\Re}~\frac{\lambda}{2}{\mathbf w^\top}\mathbf w+\frac{1}{m}\sum_{i=1}^m L_{sq}\left({y_i} - {\mathbf w^\top}\varphi \left({\mathbf x_i}\right)-b\right),
\end{equation}
where $\lambda>0$ is the regularization parameter, $\mathbf w$ is the normal of the hyperplane, $b$ is the bias, $\varphi(\mathbf x)$ is a map which maps the input $\mathbf x$ into a high-dimensional feature space, especially for managing the nonlinear learning problems, and $L_{sq}(\xi)=\frac{1}{2}\xi^2$ is the least squares loss with ${\xi}={y} - {\mathbf w^\top}\varphi \left({\mathbf x}\right)-b$ being the predict error.

By replacing $L_{sq}(\xi)$ in \eqref{eq:LSSVM} with the truncated least squares loss $L_\tau(\xi)$:
\begin{equation}\label{eq:L_tao}
L_\tau\left(\xi\right)=\frac{1}{2}\min\left(\tau^2,\xi^2\right)= \left\{\begin{array}{*{10}{c}}{\frac{1}{2}\xi^2},&\text{if}~ |\xi| \leq \tau,\\
\frac{1}{2}\tau^2, &\text{if}~ |\xi|> \tau,
\end{array}\right.
\end{equation}
Wang et al.\cite{KuainiWang2014} and Yang et al.\cite{XiaoweiYang2014} introduced the Robust LSSVM (R-LSSVM):
  \begin{equation}\label{eq:R-LSSVM}
  \min_{\mathbf w\in\Re^m,b\in\Re}~\frac{\lambda}{2}{\mathbf w^\top}\mathbf w+\frac{1}{m}\sum_{i=1}^m L_{\tau}\left({y_i} - {\mathbf w^\top}\varphi \left({\mathbf x_i}\right)-b\right),
\end{equation}
where $ \tau \geq 0 $ is the truncated parameter which controls the errors of the outliers. Fig. \ref{fig:1} plots the $L_\tau(\xi)$ in \eqref{eq:L_tao} with $\tau=1.2$, the least square loss $L_{sq}(\xi)$ and the difference between them $L_2(\xi)$. It is clear that the losses of the outliers (samples with larger errors) are bounded by $L_\tau\left(\xi\right)$, hence it reduce the effects of the outliers in R-LSSVM. We will investigate the robustness of the R-LSSVM from a re-weighted viewpoint in section \ref{sec:robust}.

\subsection{Existing algorithms for R-LSSVM}
The truncated least squares loss $L_\tau(\xi)$ is non-convex and non-smooth, which can be easily observed by Fig. \ref{fig:1}, but $L_\tau\left(\xi\right)$ can be expressed as the difference between two convex functions $L_{sq}\left(\xi\right)$ and $L_2\left(\xi\right)$ \cite{KuainiWang2014}, where
 \begin{equation}\label{eq:L1_L2}
 L_2\left(\xi\right)= \left\{\begin{array}{*{10}{l}}{0},&\text{if} ~|\xi|\leq \tau,\\
{\frac{1}{2}(\xi^2-\tau^2)}, &\text{if} ~|\xi|> \tau.
\end{array}\right.
\end{equation}
Then R-LSSVM can be transformed to a difference of convex (DC) programming:
\begin{equation}\label{eq:pR-LSSVM}
\min_{\mathbf w\in\Re^m,b\in\Re}~\frac{\lambda}{2}{\mathbf w^\top}\mathbf w+\frac{1}{m}\sum_{i=1}^m L_{sq}\left(y_i-{\mathbf w^\top}\varphi\left(\mathbf x_i\right)-b\right)-\frac{1}{m}\sum_{i=1}^m L_2{\left(y_i-{\mathbf w^\top}\varphi\left(\mathbf x_i\right)-b\right)}.
\end{equation}

\begin{figure}
\begin{minipage}{0.49\linewidth}
\includegraphics[width=1\textwidth]{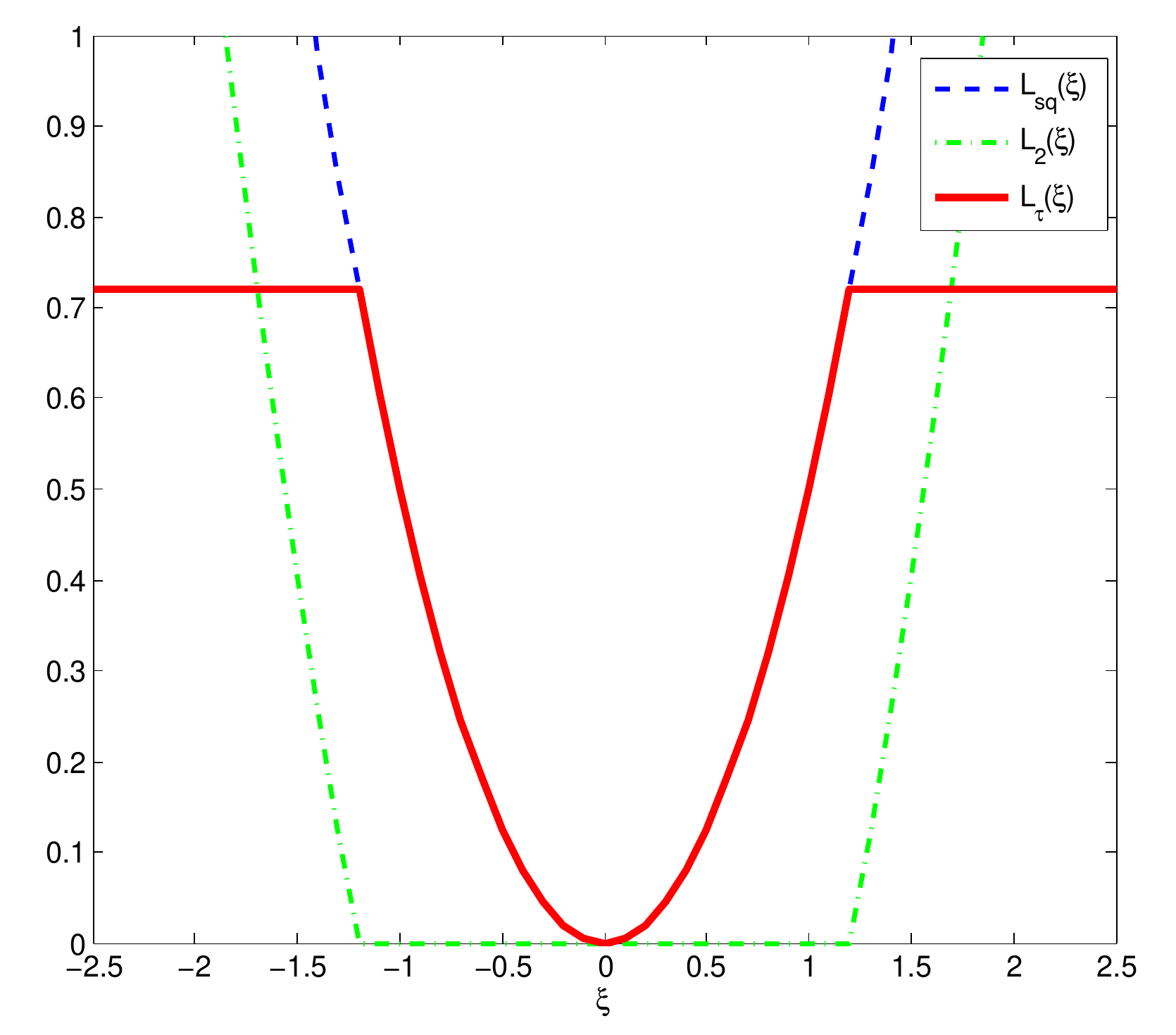}
   \caption{Plots of the least squares loss $L_{sq}(\xi)$ (dashed), the truncated least squares loss $L_\tau(\xi)$ (solid) and their difference $L_2(\xi)=L_{sq}(\xi)-L_\tau(\xi)$ (dotted-dashed), where $\tau=1.2$}\label{fig:1}
\end{minipage}
\hfill
\begin{minipage}{.49\linewidth}

\includegraphics[width=1\textwidth]{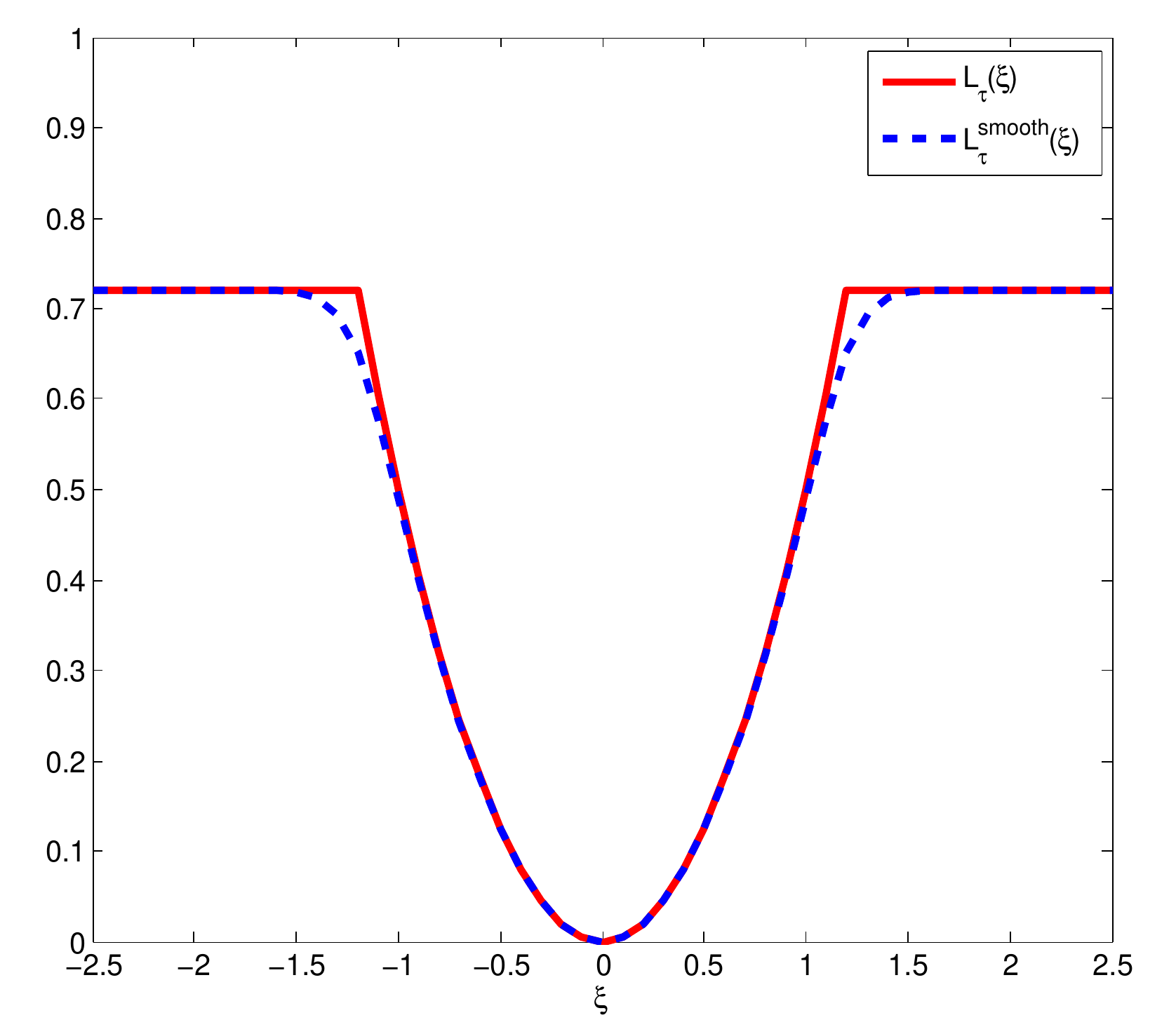}
  \caption{Plots of the truncated least squares loss $L_\tau(\xi)$ (solid) and the smoothed truncated least squares loss $L_\tau^{smooth}(\xi)$ (dashed) with $p=5$
  }\label{fig:2}
\end{minipage}
\end{figure}
%

 Wang et al.\cite{KuainiWang2014} and Yang et al.\cite{XiaoweiYang2014} solve the DC programming \eqref{eq:pR-LSSVM} by the Concave-Convex Procedure (CCCP). Then through different methods, they both focus on solving the following linear equations \eqref{eq:dr_lssvm} iteratively.
\begin{equation}\label{eq:dr_lssvm}
\left[\begin{array}{rr}
         I_m+\frac{1}{m\lambda} K&e \\
         e^\top & 0
       \end{array}
\right]\left[\begin{array}{c}\beta\\ b\end{array}\right]=\left[\begin{array}{c} \mathbf y - {\gamma^{(t)}}\\ 0\end{array}\right],
\end{equation}
where $K$ is the positive semi-definite kernel matrix satisfying $K_{ij}=k\left({\mathbf x_i},{\mathbf x_j}\right)={\varphi\left(\mathbf x_i\right)}^\top\varphi\left(\mathbf x_j\right)$, $\left(i,j\in M=\{1,\ldots,m\}\right)$, ${I_m} \in {\Re ^{m\times m}}$ is a identity matrix, $\mathbf y={\left({{y_1},\cdots,{y_m}}\right)^\top}$, $e = {( {1, \cdots ,1})}^\top\in {\Re ^{m} }$, 
and $\gamma^{(t)}={\left({{\gamma_1^{(t)}}, \cdots,{\gamma_m^{(t)}}}\right)^\top}$ is the value of $\gamma$ at the $t$-th iteration satisfying
\begin{equation}\label{eq:gamma}
\gamma_i^{(t)}\in\partial L_2(\xi_i^{(t)}), ~i=1,\cdots,m,
\end{equation}
where $\xi_i^{(t)}={y_i}-{K_{iM}}\beta^{(t)}-b^{(t)}$, $K_{iM}= \left[k\left({\mathbf x_i},{\mathbf x_1}\right),\cdots,k\left({\mathbf x_i},{\mathbf x_m}\right)\right]$ is the $i$-th row of the kernel matrix $K$.

Through iteratively solving \eqref{eq:dr_lssvm} with respect to $\beta$ and $b$ until convergence, the output deterministic function is $f(x)=\sum\limits_{i=1}^{m}\beta_i k(\mathbf x_i, \mathbf x)+b$.

In order to compute \eqref{eq:gamma}, Wang et al.\cite{KuainiWang2014} neglect the non-differentiability points in $L_2(\xi)$ and adopt the following formula:
 \begin{equation}
 \gamma_i^{(t)}=\left\{\begin{array}{*{10}{l}}{0},&\text{if} ~|\xi_i^{(t)}|\leq \tau,\\
{\xi_i^{(t)}}, &\text{if} ~|\xi_i^{(t)}|> \tau,
\end{array}\right.
\end{equation}
and Yang et al. compute \eqref{eq:gamma} after smoothing the function $L_2(\xi)$ by a piecewise quadratic function\cite{XiaoweiYang2014}.

One limitation of these two algorithms is that the solution lacks sparseness. That is because the coefficient matrix of \eqref{eq:dr_lssvm} is a nonsingular symmetric dense matrix and the vector on the right side of equations is dense. Hence the training speeds of these two algorithms are slow and they can not train large-scale problems efficiently.
\section{Robustness of R-LSSVM from a re-weighted viewpoint}
\label{sec:robust}
Wang et al.\cite{KuainiWang2014} illustrate the robustness of R-LSSVM only through experiments. Yang et al.\cite{XiaoweiYang2014} explain it from the relationship between the solutions of R-LSSVM and W-LSSVM\cite{J.A.K.Suykens2002}.
In this section, we will show that R-LSSVM enjoys the robustness from a re-weighted viewpoint\cite{Feng2016}.

By the representer theorem in section \ref{sec:4}, R-LSSVM can be translated into the following model in primal space without the implicit feature map $\varphi(\mathbf x)$:
\begin{equation}\label{eq:PR-LSSVM}
  \min_{\alpha\in\Re^m,b\in\Re}~\frac{\lambda}{2}{\alpha^\top} K  \alpha+\frac{1}{m}\sum_{i=1}^m L_{\tau}\left({y_i} - K_{iM}{\alpha}-b\right),
\end{equation}

In order to explain the robustness of the preceding model \eqref{eq:PR-LSSVM} more clearly, we propose an equivalent form of $L_{\tau}$ in Lemma \ref{lemma:1} from the idea in \cite{Geman1995}\cite{Nikolova2005}.

\begin{lemma}\label{lemma:1}
$L_\tau(\xi)=\frac{1}{2}\min\{\xi^2,\tau^2\}$ can be expressed as \begin{equation}\label{eq:eqv-trloss}
L_\tau(\xi)=\min_{\omega\in\Re_+}{\frac{1}{2}\omega \xi^2+\phi(\omega)}
\end{equation}
where
\begin{equation}\label{eq:phi}
  \phi(\omega)={\frac{\tau^2}{2}(1-\omega)_+}.
\end{equation}
\end{lemma}
\begin{proof}
\begin{eqnarray*}
\min\limits_{\omega\in\Re_{+}}{\frac{1}{2}\omega \xi^2+\phi(\omega)}
&=&\min\limits_{\omega\in\Re_{+}}\begin{cases}\frac{1}{2}(\xi^2-\tau^2)\omega+\frac{1}{2}\tau^2,~&\text{if} ~0\leq\omega\leq 1,\\
\frac{1}{2}\omega \xi^2,~&\text{if} ~\omega> 1,
\end{cases}\\
&=&\begin{cases}{\frac{1}{2}\xi^2},&\text{if}~ |\xi| \leq \tau,\\
\frac{1}{2}\tau^2, &\text{if}~ |\xi|> \tau,
\end{cases}\\
&=&L_\tau\left(\xi\right).
\end{eqnarray*}
Moreover,
\begin{equation}\label{eq:weight}
  \omega^\ast:=\mathop{\arg\min}_{\omega\in\Re_{+}}\left\{\frac{1}{2}\omega \xi^2+\phi(\omega)\right\}=\begin{cases}1,&\text{if}~ |\xi| \leq \tau,\\
0, &\text{if}~ |\xi|> \tau.
\end{cases}
\end{equation}
\end{proof}
By Lemma \ref{lemma:1} and the research of re-weighted LSSVM in \cite{Brabanter2009}, we have
\begin{proposition}
Any stationary point of R-LSSVM \eqref{eq:PR-LSSVM} can be obtained by solving an iteratively re-weighted LSSVM as follows:
\begin{equation}\label{eq:rewighted_lssvm}
  \min_{\alpha\in\Re^m,b\in\Re}~\frac{\lambda}{2}{\alpha^\top K \alpha}+\frac{1}{2m}\sum_{i=1}^m \omega_i^{(t)}\left({y_i} - K_{iM}\alpha-b\right)^2,
\end{equation}
where $\omega_i^{(t)}$ is the value of $t$-th iteration of the weight $\omega_i$.
\end{proposition}
\begin{proof}
Substituting \eqref{eq:eqv-trloss} into \eqref{eq:PR-LSSVM}, we have
\begin{equation}\label{eq:inproof}
\min_{\omega\in{\Re_+^m},\alpha\in{\Re^m},b\in\Re}~J(\alpha,b,\omega):=\frac{\lambda}{2}\alpha^\top K \alpha+\frac{1}{m}\sum_{i=1}^{m}{\frac{1}{2}\omega_i \xi_i^2+\frac{1}{m}\sum_{i=1}^{m}\phi(\omega_i)},
\end{equation}
where $\xi_i=y_i-K_{iM}\alpha-b$. Since $J(\alpha,b,\omega)$ is nonconvex, only a stationary point of preceding minimization problem can be expected.
Let $(\alpha^\ast, b^\ast)$ be one of the stationary points of \eqref{eq:PR-LSSVM}. By the analysis above, there exists $\omega^*\in\mathop{\arg\min}_{\omega\in{\Re_+^m}}~J(\alpha^*,b^*,\omega)$ such that $(\alpha^\ast, b^\ast, \omega^*)$ be the solution of \eqref{eq:inproof}. On the other hand, if $(\alpha^\ast, b^\ast, \omega^*)$ is any stationary point of \eqref{eq:inproof}, then $(\alpha^\ast, b^\ast)=\mathop{\arg\min}_{\alpha\in{\Re^m}, b\in\Re}~J(\alpha,b,\omega^*)$ also solves \eqref{eq:PR-LSSVM}. Hence, we can iteratively solve \eqref{eq:inproof} by alternating direction method (ADM)\cite{He2012} 
as follows:
\begin{eqnarray}
(\alpha^{(t)},b^{(t)})&=&\mathop{\arg\min}_{\alpha\in{\Re^m},b\in\Re}~J(\alpha,b,\omega^{(t-1)})\label{eq:1}\\
\omega^{(t)}&\in&\mathop{\arg\min}_{\omega\in{\Re_+^m}}~J(\alpha^{(t)},b^{(t)},\omega)\label{eq:2}
\end{eqnarray}
Obviously, the optimization problem in \eqref{eq:2} has the closed form solution as \eqref{eq:weight}. The optimization problem in \eqref{eq:1} is just the re-weighted LSSVM \eqref{eq:rewighted_lssvm}.
\end{proof}

Since $\xi_i$ denotes the predicted error, similar to the robustness analysis in article \cite{Feng2016}, the larger $|\xi_i|$ is, the more likely that the instance pair $(\mathbf x_i,y_i)$ tends to be an outlier. From \eqref{eq:weight} and \eqref{eq:rewighted_lssvm}, it observes that when the $|\xi_i|$ is sufficiently large for the outlier instance $(\mathbf x_i,y_i)$, the corresponding weight $\omega_i$ in \eqref{eq:rewighted_lssvm} will be 0. That is, the truncated least squares loss function $L_\tau$ can reduce the influence of samples which are far away from their true targets. This explains the robustness of R-LSSVM from the re-weighted viewpoint.
\section{Sparse R-LSSVM algorithm}
In this section, we give the primal R-LSSVM and propose the sparse algorithm to obtain the sparse solution of the R-LSSVM.


\subsection{Primal R-LSSVM}
\label{sec:4}
If loss function is convex such as in LSSVM model \eqref{eq:LSSVM}, by duality theory, the optimal solution $\mathbf w$ can be represented as
\begin{equation}\label{eq:w}
  \mathbf w=\sum_{i=1}^ m{\alpha_i}\varphi\left(\mathbf x_i\right),
\end{equation}
where $\alpha_i\in\Re$. If loss function is nonconvex, the strong duality does not hold, hence we cannot get \eqref{eq:w} by duality. However, by the representer theorem in \cite{Scholkopf2001}\cite{Shai2014}, it is easily to prove that \eqref{eq:w} also holds.
\begin{theorem}
 Assume that $\varphi$ is a mapping from $\Re^l$ to a Hilbert space. Then there exists a vector $\alpha\in\Re^m$ such that \eqref{eq:w}
is an optimal solution of \eqref{eq:R-LSSVM} and \eqref{eq:pR-LSSVM}.
\end{theorem}

Substituting \eqref{eq:w} into \eqref{eq:pR-LSSVM}, we get a DC programming with regard to $\alpha$ and $b$ as follows:
\begin{equation}\label{eq:vex_cave}
{\min \limits_{\alpha\in\Re^m,b\in\Re}~H{\left(\alpha,b\right)}={H_{1}\left(\alpha,b\right)}-{H_{2}\left(\alpha,b\right)}}
\end{equation}
with convex functions $H_{1}\left(\alpha,b\right)={\frac{\lambda}{2}}{\alpha^\top}K{\alpha}+\frac{1}{m}\sum\limits_{i=1}^mL_{sq}\left({y_i}-{K_{iM}}\alpha-b\right)$
and
$H_{2}\left(\alpha,b\right)= \frac{1}{m}\sum\limits_{i=1}^m {{L_2} \left({y_i}-{K_{iM}}\alpha-b\right)}$. We call the model \eqref{eq:vex_cave} or its equivalent form \eqref{eq:PR-LSSVM} as primal R-LSSVM for convenience.

Using CCCP method in \cite{KuainiWang2014}\cite{XiaoweiYang2014}\cite{Yuille2003}, the solution to the problem \eqref{eq:vex_cave} can be obtained by iteratively solving the following convex QP until it converges:
\begin{equation}\label{eq:alphab_t+1}
\begin{split}
{\left(\alpha ^{\left(t + 1\right)}, b^{\left(t+1\right)}\right)}
&=\mathop{\arg\min}_{\alpha\in\Re^m,b\in\Re}\left\{H_{1}\left(\alpha,b\right)-\langle [\alpha^\top,b]^\top,\partial H_2(\alpha^{\left(t\right)},b^{\left(t\right)})\rangle\right\}\\
&=\mathop{\arg\min}_{\alpha\in\Re^m,b\in\Re}\left\{H_{1}\left(\alpha,b\right)+\frac{1}{m}\sum\limits_{i=1}^m{\gamma_i^{(t)}}\left(K_{iM}\alpha+b\right)\right\},
\end{split}
\end{equation}
where $\gamma_i^{(t)}$ is the same as \eqref{eq:gamma} with $\xi_i^{(t)}=y_i-K_{iM}\alpha^{(t)}-b^{(t)}$.

However, the computation of $\gamma_i^{(t)}$ is not simple, since $L_2\left(\xi\right)$ is non-differentiable at some points. Inspired by the idea in \cite{ShuishengZhou2013}, we smooth $L_2\left(\xi\right)$ by the entropy penalty function. Let
\begin{equation}\label{eq:L_2*}
\bar L_2\left(\xi\right)=\frac{1}{2}\max\left\{0,\xi^2-\tau^2\right\}+\frac{1}{2p}\log\left(1+\exp\left(-p|\xi^2-\tau^2|\right)\right),
\end{equation}
then we have $\bar L_2\left(\xi\right)\rightarrow L_2{\left(\xi\right)}$ whenever $p\rightarrow+\infty$. $\bar L_2\left(\xi\right)$ is the smooth approximation of $L_2{\left(\xi\right)}$, and the upper bound of the difference between $\bar L_2\left(\xi\right)$ and $L_2{\left(\xi\right)}$ is $\frac{\log2}p$. In practice, if we set $p$ sufficiently large such as $p=10^4$, the difference between them can be neglected. Fig. \ref{fig:2} shows the comparison between $L_\tau(\xi)$ and the smoothed truncated least squares loss function $L_\tau^{smooth}(\xi)=L_{sq}(\xi)-\bar L_2\left(\xi\right)$ with $p=5$.

 The derivative of $\bar L_2\left(\xi\right)$ is:
\begin{equation}\label{eq:L_2*_grad}
\bar L_2'\left(\xi\right)=\xi\cdot\frac{\min\left\{1,\exp\left[{p\left(\xi^2-\tau^2\right)}\right]\right\}}{1+\exp{\left(-p\left|\xi^2-\tau^2\right|\right)}}.
\end{equation}
Replacing $L_2\left(\xi_i\right)$ with $\bar L_2\left(\xi_i\right)$ in \eqref{eq:gamma},
 the $\gamma_i^{\left(t\right)}$ in \eqref{eq:alphab_t+1} is calculated as follows:
\begin{equation}\label{eq:gamma_2}
\gamma_i^{(t)} =\frac{\xi_i^{(t)}\min\left\{1,\exp\left[p\left({\xi _i^{(t)}}^2-\tau^2\right)\right]\right\}}{1+\exp{\left(-p\left|{\xi_i^{(t)}}^2-\tau^2\right|\right)}}
\end{equation}

Yang et al.\cite{XiaoweiYang2014} also adopt a smooth procedure, but their method has to tune the smoothing parameter to get the best effect. That makes the parameter adjustment procedure complex. In comparison, our smoothing strategy based on entropy penalty function does not need to tune such parameter. What we need to do is set a large value for $p$ in \eqref{eq:gamma_2}.
\subsection{Sparse solution for Primal R-LSSVM}
After obtaining $\gamma_i^{(t)}$ by \eqref{eq:gamma_2}, $(\alpha^{(t+1)}, b^{(t+1)})$ in \eqref{eq:alphab_t+1} are the solutions of the following system of linear equations:
\begin{equation}\label{eq:pr_lssvm}
\left[\begin{array}{rc}
         m\lambda K+KK^\top & Ke \\
         e^\top K^\top & m
       \end{array}
\right]\left[\begin{array}{c}\alpha\\ b\end{array}\right]=\left[\begin{array}{c} K\\ e^\top\end{array}\right]\left(\mathbf y - {\gamma^{(t)}}\right).
\end{equation}

It seems that \eqref{eq:pr_lssvm} is more complicated than \eqref{eq:dr_lssvm} in a first sight. However, the coefficient matrix of \eqref{eq:dr_lssvm} is nonsingular symmetric dense matrix, which leads to a  non-sparse solution of \eqref{eq:dr_lssvm}. In comparison, the coefficients matrix of \eqref{eq:pr_lssvm} may be low rank if the related kernel matrix $K$ is low rank or is approximated by a low rank matrix. In this situation, \eqref{eq:pr_lssvm} may have sparse solution, which overcomes the limitation of the previous methods partly.

Now, we discuss the sparse optimization solution of \eqref{eq:pr_lssvm} as soon as the kernel matrix can be approximated by a low rank matrix.

After simply calculation, we get the bias $b=\frac{1}{m}\left[e^\top(\mathbf y - {\gamma^{(t)}})-e^\top K \alpha\right]$ by \eqref{eq:pr_lssvm}. Eliminating $b$, \eqref{eq:pr_lssvm} is simplified to the following linear equation:
\begin{equation}\label{eq:alpha_dense}
(m\lambda K+KK-\tfrac{1}{m}Kee^\top K)\alpha=K\left[(\mathbf y - {\gamma^{(t)}})-\tfrac{e^\top(\mathbf y - {\gamma^{(t)}})}{m}e\right].
\end{equation}

Nystr\"{o}m Approximation is a most popular method to obtain the low-rank approximation of kernel matrix $K$ (see \cite{Williams2001}\cite{Petros2005}\cite{Zhang2010}\cite{Si2016} and the references therein). The low-rank approximation method is not the point of this paper. For simplicity, we employ Zhou's pivoting Cholesky factorization method\cite{sszhou2016}. Let $B\subset M$ corresponding to the indices of the $r$ landmark points, $K_{MB}\in\Re^{m\times r}$ be the sub-matrix of $K$ whose elements are $K_{ij}=k\left({\mathbf x_i},{\mathbf x_j}\right)$ for $i\in M$ and $j\in B$, and $K_{BB}\in\Re^{r\times r}$ be defined similarly. By the pivoting Cholesky factorization method in \cite{sszhou2016}, we can obtain the full column rank matrix $P\in\Re^{m\times r}$ satisfying $PP^\top=K_{MB}K_{BB}^{-1}K_{MB}^\top$ as the best rank-$r$ Nystr\"{o}m type approximation of $K$ under the trace norm, and in all process only the selected $r$ columns and the diagonal of the kernel matrix are necessary. If $K_{MB}$ is gotten by some other low-rank approximation methods \cite{Williams2001}\cite{Petros2005}\cite{Zhang2010}\cite{Si2016}, let $P=K_{MB}K_{BB}^{-\frac{1}{2}}$ and the following analysis is the same.

Substituting $PP^\top$ into \eqref{eq:alpha_dense} instead of $K$, \eqref{eq:alpha_dense} is simplified as:
\begin{equation}\label{eq:alpha_dense2}
(m\lambda I_r+P^\top P-\tfrac{1}{m}P^\top ee^\top P)P^\top \alpha=P^\top\left[(\mathbf y - {\gamma^{(t)}})-\tfrac{e^\top(\mathbf y - {\gamma^{(t)}})}{m}e\right],
\end{equation}
where $I_r \in \Re^{r\times r}$ is a identity matrix. By permuting rows of matrix $P$, we get $[P_B^\top, P_N^\top]^\top$, where $P_B\in \Re^{r\times r}$ is a full rank matrix (and will be a lower triangular matrix if $P$ is obtained as \cite{sszhou2016}, hence $P_B^{-1}$ is computed with cost $O(r^2)$ instead of $O(r^3)$), and $P_N$ is comprised by the rest $m-r$ rows of $P$. Correspondingly, let $\alpha=[\alpha_B^\top, \alpha_N^\top]^\top$, then we have
\begin{equation}\label{eq:alpha_Chol_BN}
\left\{\begin{split}
&{\alpha _B} = \left(P_{B}^\top\right)^{-1}J^{-1}{P^\top}\left(\mathbf y - {\gamma^{\left(t\right)}}-\frac{e^\top \left(\mathbf y -{\gamma^{(t)}}\right)}me\right),\\
&{\alpha _N}=0
\end{split}\right.
\end{equation}
is the sparse solution of \eqref{eq:alpha_dense2}, where
\begin{equation}\label{eq:J}
J= {\left(m\lambda {I_r} + {P^\top}P\right)-\frac{1}m\left(e^\top P\right)^\top e^\top P}.
\end{equation}
So the sparse R-LSSVM (SR-LSSVM) algorithm is obtained by iteratively updating $(\alpha^{(t+1)}, b^{(t+1)})$ as follows:
\begin{numcases}{}
{\alpha _B^{(t+1)}} = \left(P_{B}^\top\right)^{-1}J^{-1}{P^\top}\left(\mathbf y - {\gamma^{\left(t\right)}}-\frac{e^\top \left(\mathbf y -{\gamma^{(t)}}\right)}me\right),\label{eq:alpha_Chol}\\
{\alpha _N^{(t+1)}}=0,\nonumber\\
b^{(t+1)}=\frac{1}m\left[e^\top\left(\mathbf y - {\gamma^{(t)}}\right)-e^\top PP_B^\top\alpha _B^{(t+1)}\right].\label{eq:b_Chol}
\end{numcases}
%
\subsection{Sparse R-LSSVM algorithm}

From the above analysis, our SR-LSSVM algorithm is listed as Algorithm \ref{alg:1}.

\begin{algorithm}[!h]
\renewcommand{\algorithmicrequire}{\textbf{Input:}}
\renewcommand\algorithmicensure {\textbf{Output:} }
\linespread{1.1}\selectfont
\caption{\textbf{\emph{SR-LSSVM}}---\textbf{S}parse \textbf{R-LSSVM}}
\label{alg:a1}
\begin{algorithmic}[1]\label{alg:1}
\REQUIRE Training set $\mathbf T=\left\{\mathbf x_i,y_i\right\}_{i=1}^m $, kernel matrix $K$, the regularization parameter $\lambda>0$, the truncated parameter $\tau$, the stop criterion $\varepsilon>0$, $r=|B|$.
\ENSURE $\alpha_B$ and $b$.
\STATE Find $P$ and $B$ such that $K\approx PP^\top$, 
$\gamma^{(0)}=0\in\Re^m$;
\STATE Compute $J$ as \eqref{eq:J}. Set $t=0$;
\STATE Update $\alpha_B^{(t+1)}$ and $b^{(t+1)}$ by \eqref{eq:alpha_Chol} and \eqref{eq:b_Chol};
\STATE Set $\xi^{(t+1)}={\mathbf y}-PP_B^\top\alpha_B^{(t+1)}-b^{(t+1)}$, and compute ${\gamma_i^{(t+1)}}$ by \eqref{eq:gamma_2};
\IF{$\left\| {{\gamma^{(t+1)}} - {\gamma^{(t)}}} \right\| < \varepsilon $}
\STATE stop with $\alpha_B=\alpha_B^{(t+1)}$ and $b=b^{(t+1)}$;
\ELSE
\STATE let $t=t+1$, go to step 3.
\ENDIF
\end{algorithmic}
\vspace*{-3pt}
\end{algorithm}

%
%
%
%
%
%
%

After obtaining the optimal $\alpha_B$ and $b$ by Algorithm \ref{alg:1}, the decision function for regression is:
\begin{equation}\label{eq:decision_function}
f(\mathbf x) = \sum\limits_{i \in B} {{\alpha _i}k\left({\mathbf x_i},\mathbf x\right)}+b.
\end{equation}
For classification, the decision function is $\texttt{sgn}{f(\mathbf x)}$.
We give some comments about Algorithm \ref{alg:1}.

\emph{\textbf{Comment 1}}. If let $\gamma^{(0)}=0$ as the starting point, the first cycle of Algorithm \ref{alg:1} is equivalent to solving the primal LSSVM (P-LSSVM) problem\cite{sszhou2016}.

\emph{\textbf{Comment 2}}. In equation \eqref{eq:alpha_Chol}, if we set $$\upsilon^{(t)}=J^{-1}{P^\top}\left(\mathbf y - {\gamma^{\left(t\right)}}-\frac{e^\top \left(\mathbf y -{\gamma^{(t)}}\right)}me\right),$$ then ${\alpha _B^{(t+1)}} = \left(P_{B}^\top\right)^{-1}\upsilon^{(t)}$ and $b^{(t+1)}=\frac{1}m\left[e^\top\left(\mathbf y - {\gamma^{(t)}}\right)-e^\top P\upsilon^{(t)}\right]$, so $\xi^{(t+1)}={\mathbf y}-P\upsilon^{(t)}-b^{(t+1)}$. In step 3 of Algorithm \ref{alg:1}, we compute $\upsilon^{(t)}$ instead of $\alpha_B^{(t+1)}$, and the cost of step 3 is decreased further. The output $\alpha_B$ can only be calculated at the last round by $\upsilon^{(t)}$ .

\emph{\textbf{Comment 3}}. To promote computational efficiency, Equation \eqref{eq:alpha_Chol} can be rewritten as:
\begin{equation}\label{eq:alpha_sparse}
\alpha_B^{(t+1)}=\alpha_{LS}-G\left(P_{S_{t}}^\top\gamma^{(t)}_{S_{t}}-\frac{e^\top \gamma^{(t)}}{m}\widehat{P}\right),
\end{equation}
where $G=\left(P_{B}^\top\right)^{-1}J^{ - 1}$, $\widehat{P}=P^\top e$, $\alpha_{LS}=G({P^\top}\mathbf y-\frac{e^\top \mathbf y}{m}\widehat{P})$ is the sparse solution of primal LSSVM, $S_t\subset M$ is the index set of nonzero elements of $\gamma^ {(t)}$, $P_{S_{t}}$ is comprised by several rows of $P$, and the indexes of these rows in $P$ correspond to the elements in $S_t$, $\gamma^{(t)}_{S_{t}}$ is a vector comprised of nonzero elements of $\gamma^ {(t)}$.

Then the step 2 and 3 in Algorithm \ref{alg:1} can be replaced with the following:

Step 2': Compute $J,~G,~\widehat{P}$ and $\alpha_{LS}$. Set $t=0$;

Step 3': Update $\alpha_B^{(t+1)}$ and $b^{(t+1)}$ by \eqref{eq:alpha_sparse} and \eqref{eq:b_Chol} respectively.

\emph{\textbf{Comment 4}}. In Algorithm \ref{alg:1}, the parameter $\tau$ limits the upper bound of loss function. $\tau$ should not be set too large or small. The improper $\tau$ results in poor generalization performance. To overcome the sensitivity of the loss function to $\tau$, we can tone $\tau$ as follows. Firstly, set a little larger $\tau$, such as $\tau=\delta*\max \{|\xi_1|, \cdots, |\xi_m|\}$, where $0<\delta<1$. Then add the following step between the step 3 and step 4 in Algorithm \ref{alg:1}: reduce $\tau$ if $\left\| {{\gamma^{(t+1)}} - {\gamma^{(t)}}} \right\|$ is small until $\tau\leq \tau_{\min}$, where $ \tau_{\min}$ is the minimum of $\tau$ we set.


\subsection{Convergence and Complexity analysis}
 CCCP is globally or locally convergent, see \cite{Yuille2003} \cite{Tao2014} \cite{Bharath2012}. 
Similar to the convergence proof of DCA (DC Algorithm) for general DC programs in article \cite{PHAM1997}, we have the following Lemma.
 \begin{lemma}\label{lemma:convergence}
  If the optimal value of the problem \eqref{eq:vex_cave} is finite, and the infinite sequences $(\alpha^{(t)},b^{(t)})$ and ${\partial H_2(\alpha^{(t)},b^{(t)})}$ are bounded, then every limit point $(\widetilde{\alpha},\widetilde{b})$ of the sequence $(\alpha^{(t)},b^{(t)})$ is a generalized KKT point of $H_1(\alpha,b)-H_2(\alpha,b)$.
  \end{lemma}
Obvious, the objective function of \eqref{eq:vex_cave} and \eqref{eq:PR-LSSVM} is bounded below. Assume the prediction error variable $\xi_i^{(t)}$ is bounded, which is reasonable in real application, then $\gamma_i^{(t)}$ is bounded by  \eqref{eq:gamma_2}. So $(\alpha^{(t)},b^{(t)})$ and $\partial H_2(\alpha^{(t)},b^{(t)})= [{\gamma^{(t)}}^\top K, {\gamma^{(t)}}^\top e]^\top$ are also bounded because of the boundedness of $(P_B^\top)^{-1}$, $J^{-1}$ and $P^\top$ in \eqref{eq:alpha_Chol} and \eqref{eq:b_Chol}. By Lemma \ref{lemma:convergence}, we get the following theorem.
\begin{theorem}Assume the predict error ${\xi}={y} - {\mathbf w^\top}\varphi \left({\mathbf x}\right)-b$ is bounded for all given samples $(\mathbf x,y)$ with selected parameters $\mathbf w$ and $b$, then limit point of the sequence $(\alpha^{(t)},b^{(t)})$ is the generalized KKT point of the problem \eqref{eq:vex_cave}, that is, Algorithm \ref{alg:1} is convergent.
\end{theorem}

For Algorithm \ref{alg:1}, the computation cost of step 1 and step 2 are both $O\left(mr^2\right)$ $(r\ll m)$\cite{sszhou2016}. The complexity of iteratively solving step 3 is $O\left(T_smr\right)$, where $T_s$ is the total iterative number of SR-LSSVM. So the overall complexity of this algorithm is $O\left(mr^2 +T_smr\right)$. If we utilize the technique in comment 3 to compute $\alpha_B^{(t+1)}$, then the complexity of step $3$ in Algorithm \ref{alg:1} reduced to $O\left(|S_t|r\right)$ $(|S_t|<m)$ which is the complexity of step $3'$. In comparison, the computational complexities of Wang's and Yang's R-LSSVM algorithms in \cite{KuainiWang2014} and \cite{XiaoweiYang2014} are both $ O\left(T_dm^3\right)$, where $T_d$ is the iterations of their algorithms. It is obvious that our method has smaller computational complexity than existing approaches.

 \textbf{Parallel computing potential.} In the Algorithm 1, some calculations are easy to perform, so serial computing is enough for them. However, for some costly calculations, we can utilize parallel computing to further improve computing efficiency. The main computational cost of Algorithm 1 is from computing $P^\top P$, which can be implemented in parallel. For example, $P$ can be partitioned into $k$ chunks according to row satisfying $P^\top=[P_1^\top,\ldots,P_k^\top]$, so $P^\top P=\sum\limits_{i=1}^k P_i^\top P_i$ which can be efficiently calculated by the parallel algorithm of matrix multiplication, where $P_i$ is the $i$th block of the matrix $P$.
 \section{Numerical experiments and discussions}
To examine the validity of the proposed algorithm, we compare our SR-LSSVM with the R-LSSVM-W\cite{KuainiWang2014} (Wang's algorithm for R-LSSVM), R-LSSVM-Y\cite{XiaoweiYang2014} (Yang's algorithm for R-LSSVM), the classical LSSVM, W-LSSVM\cite{J.A.K.Suykens2002}, the FS-LSSVM\cite{Suykens2002} which is operated in the LS-SVMlab v1.8 software \cite{Brabanter2011_lssvmtoolbox}\footnote[1]{Codes are available in http://www.esat.kuleuven.be/sista/lssvmlab/.} and the SVMs (C-SVC for classification and $\epsilon$-SVR for regression) which are implemented in the LIBSVM software\footnote[2]{Codes are available in https://www.csie.ntu.edu.tw/~cjlin/libsvm/.} for medium datasets. For some large-scale problems, we only compare the proposed algorithm with some sparse algorithms, such as PCP-LSSVM\cite{sszhou2016}\footnote[3]{Codes and article can be downloaded from http://web.xidian.edu.cn/sszhou/paper.html}, FS-LSSVM\cite{Suykens2002}, Cholesky with side information (CSI)\cite{Bach_csi2005}\footnote[4]{Codes are available in http://www.di.ens.fr/$\sim$fbach/csi/index.html.} and C-SVC for classification or $\epsilon$-SVR for regression, since the others can not apply in this case.

All computations are implemented in windows 8 with Matlab R2014a. The whole experiments are run on a PC with an Intel Core i5-4210U CPU and a maximum of 8G bytes of memory available for all processes.


We fixed the values of smoothing parameter $p=10^{4}$ in SR-LSSVM and the stop criterion $\varepsilon=10^{-2}$ respectively. For all the data sets, we use cross-validation procedure and grid search to search the best values of the parameter $\lambda, \sigma, \tau$ and $h$, where $\sigma$ is the parameter in Gaussian kernel function $k(\mathbf x_i, \mathbf x_j)=\exp(-\sigma\|\mathbf x_i-\mathbf x_j\|^2)$, and $h$ is the smooth parameter in method R-LSSVM-Y.

For R-LSSVM-W and R-LSSVM-Y, the running time in our article is much less than those in \cite{KuainiWang2014}\cite{XiaoweiYang2014} for the same data sets and the total complexity is reduced from $O(T_dm^3)$ \cite{KuainiWang2014}\cite{XiaoweiYang2014}  to $O(m^3+T_dm^2)$, where the coefficient matrix of \eqref{eq:dr_lssvm} is decomposed by Cholesky factorization once and such decomposition is unchanged per loop in our experiments.

\subsection{Classification experiments}
In this section, we test one synthetic classification data set and some benchmark classification data sets to illustrate the effectiveness of the SR-LSSVM. For benchmark datasets, each attribute of the samples is normalized into $[-1,~1]$, and these datasets are separated into two groups: the medium size datasets group and the large-scale datasets group. All of them are downloaded from \cite{lib}. The experimental results on Adult data set show the reason why we separate these data sets into two groups. Finally, we test the robustness of our proposed algorithm for large-scale data sets with outliers on Cod-RNA dataset. Outliers are generated by the following procedure. We choose 30\% of samples which are far from decision hyperplane, then randomly sample 1/3 of them and flip their labels to simulate outliers.
\subsubsection{Synthetic classification dataset experiment}
\begin{figure}[!t]
 \centering
 \includegraphics[width=1\textwidth]{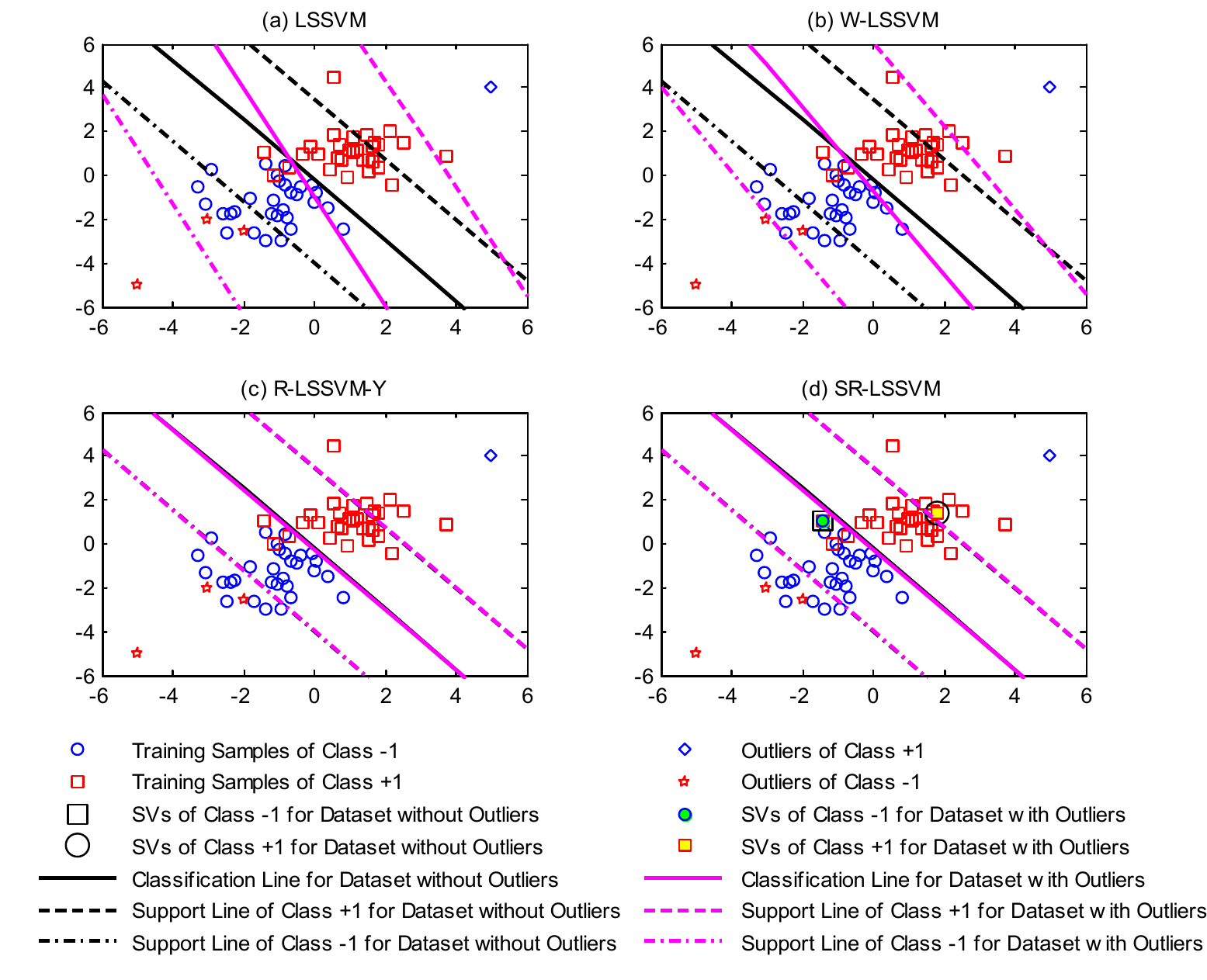}
 \caption{Comparison of the proposed approach SR-LSSVM with LSSVM, W-LSSVM and R-LSSVM-Y for linearly inseparable classification dataset with and without outliers. The numbers of the support vectors are both 2 for datasets with and without outliers for (d) SR-LSSVM. We do not mark SVs in the subgraphs (a)-(c), because almost all of training samples are SVs for LSSVM, W-LSSVM and R-LSSVM-Y. For dataset without outliers, the test accuracies are all 91.50\% for these four algorithms, and for dataset with outliers, the test accuracies are 89.50\%, 90.00\%, 91.50\% and 91.50\% for LSSVM, W-LSSVM, R-LSSVM-Y and SR-LSSVM respectively}
 \label{fig:class_normal}
\end{figure}

To compare the robustness and spareness of four algorithms LSSVM, W-LSSVM, R-LSSVM-Y and SR-LSSVM, we conduct an experiment on a linear binary classification data set including 60 training samples and 100 testing samples. Fig. \ref{fig:class_normal} shows the experimental results. To simulate outliers, we add 4 training samples labeled with wrong classes. They are marked as '$\diamond$' and '$\ast$' for positive and negative classes respectively. Through grid search, we obtain the best parameter values for this data set are $m\lambda=10^{-2}, \tau=1.5, h=0.01$.

Fig. \ref{fig:class_normal} illustrates that the decision lines of algorithms LSSVM and W-LSSVM change greatly and these two methods have lower accuracies than SR-LSSVM and R-LSSVM-Y after adding outliers. In contrast, the decision boundaries of SR-LSSVM and R-LSSVM-Y are almost unchanged and the accuracies of these two approaches remain stable before and after adding outliers. So SR-LSSVM is insensitive to outliers. Moreover, almost all of training samples are SVs for LSSVM, W-LSSVM and R-LSSVM-Y. By contrast, for SR-LSSVM, the support vector sizes are both only 2 for data sets with and without outliers. So the proposed algorithm is sparseness, which can accelerate the training speed of our approach in processing large scale problems.

\subsubsection{Medium-scale benchmark classification datasets experiments}
\begin{table}[!htp]\scriptsize
\centering
\renewcommand\arraystretch{0.85}
\caption{Comparison of the numbers of iterations, training time (seconds), mean number of support vectors (denoted by nSVs) and accuracies (\%) of different algorithms on benchmark classification data sets with outliers (10\%). The standard deviations are given in brackets. '-' means this parameter is not used by this method. The best values are highlighted in bold.}
\begin{threeparttable}
\setlength{\tabcolsep}{1.3pt}
\begin{tabular}{llrrrrrrrr}
\hline
Data &Algorithms&$m\lambda$&$\sigma$&$\tau$&$h$&Iterations&Training&nSVs&Accuracies(\%)\\
(Train,&&&&&&&Time(s)\\Test)\\
\hline
Pendigits$^1$&C-SVC&$10^{-2}$&$2^{-4}$&-&-&-&0.36(0.02)&433.5(7.7)&99.95(0.001)\\
(1466,&LSSVM&$10^{-3}$&$2^{-4}$&-&-&-&0.12(0.01)&1466($0$)&99.26(0.005)\\
733)&W-LSSVM&$10^{-3}$&$2^{-4}$&-&-&-&0.18(0.01)&1464.6(1.7)&99.92(0.002)\\
$l=16$&FS-LSSVM&$10^{-3}$&$2^{-4}$&-&-&-&0.19(0.01)&\textbf{73}(0)&99.90(0.001)\\
&R-LSSVM-W&$10^{-3}$&$2^{-4}$&1.5&-&16.1(1.7)&0.22(0.01)&1466(0)&99.37(0.003)\\
&R-LSSVM-Y&$10^{-3}$&$2^{-4}$&1.5&0.25&12.2(1.5)&0.19(0.01)&1436.8(8.0)&99.09(0.005)\\
&SR-LSSVM&$10^{-3}$&$2^{-4}$&1.5&-&\textbf{8.5}(0.9)&\textbf{0.03}($0.00$)&\textbf{73}($0$)&\textbf{99.96}(0.001)\\
\hline
Protein$^2$&C-SVC&$10^{0}$&$2^{-5}$&-&-&-&55.64(0.47)&5486.8(27.1)&77.98($0.004$)\\
(8186,&LSSVM&$10^{0}$&$2^{-5}$&-&-&-&22.67(0.30)&8185.9(0.32)&78.22(0.002)\\
3509)&W-LSSVM&$10^{0}$&$2^{-5}$&-&-&-&27.77(0.30)&8185.7(0.67)&\textbf{78.24}(0.003)\\
$l=357$&FS-LSSVM&$10^{0}$&$2^{-5}$&-&-&-&27.55(0.61)&\textbf{408.1}(1.20)&77.00(0.003)\\
&R-LSSVM-W&$10^{0}$&$2^{-5}$&0.8&-&34.9(4.2)&28.56(1.10)&8184.8(0.92)&77.81(0.023)\\
&R-LSSVM-Y&$10^{0}$&$2^{-5}$&0.8&0.7&16.7(3.4)&25.17(0.81)&7876.7(46.6)&78.23(0.004)\\
&SR-LSSVM&$10^{0}$&$2^{-5}$&0.8&-&\textbf{6}(0)&\textbf{11.65}($0.03$)&409($0$)&78.04(0.002)\\
\hline
Satimage$^3$&C-SVC&$10^{0}$&$2^{-1}$&-&-&-&\textbf{0.25}(0.01)&693.5(13.6)&99.86($0.001$)\\
(2110,&LSSVM&$10^{0}$&$2^{-1}$&-&-&-&0.76(0.01)&2109.6(0.7)&99.23(0.002)\\
931)&W-LSSVM&$10^{0}$&$2^{-1}$&-&-&-&0.90(0.02)&1897.5(2.1)&99.91(0.001)\\
$l=36$&FS-LSSVM&$10^{0}$&$2^{-1}$&-&-&-&0.50(0.02)&\textbf{105}(0)&97.93(0.008)\\
&R-LSSVM-W&$10^{0}$&$2^{-1}$&0.5&-&17.1(1.5)&0.88(0.03)&2107.3(1.2)&99.90(0.001)\\
&R-LSSVM-Y&$10^{0}$&$2^{-1}$&0.5&0.3&11.7(1.2)&0.87(0.03)&1916.5(15.1)&99.88(0.001)\\
&SR-LSSVM&$10^{0}$&$2^{-1}$&0.5&-&\textbf{6}(2.7)&0.27($0.00$)&\textbf{105}($0$)&\textbf{99.97}(0.001)\\
\hline
USPS$^4$&C-SVC&$10^{0}$&$2^{-7}$&-&-&-&\textbf{0.92}(0.02)&646.9(14.4)&99.34($0.001$)\\
(2199,&LSSVM&$10^{0}$&$2^{-7}$&-&-&-&2.53(0.01)&2198.7(0.5)&99.34(0.002)\\
623)&W-LSSVM&$10^{0}$&$2^{-7}$&-&-&-&2.67(0.01)&1973.8(3.1)&99.49(0.001)\\
$l=256$&FS-LSSVM&$10^{0}$&$2^{-7}$&-&-&-&1.21(0.01)&109(0)&98.28(0.006)\\
&R-LSSVM-W&$10^{0}$&$2^{-7}$&1.1&-&9.3(0.84)&2.60(0.02)&2193.7(2.7)&\textbf{99.52}(0.000)\\
&R-LSSVM-Y&$10^{0}$&$2^{-7}$&1.1&0.15&10.2(0.92)&2.60(0.02)&2002.4(14.9)&\textbf{99.52}(0.000)\\
&SR-LSSVM&$10^{0}$&$2^{-7}$&1.1&-&\textbf{6}(0.79)&1.91($0.01$)&\textbf{108.7}($0.3$)&\textbf{99.52}(0.000)\\
\hline
Splice&C-SVC&$10^{0}$&$2^{-9}$&-&-&-&\textbf{0.12(0.01)}&820.8(8.3)&76.38($0.08$)\\
(1000,&LSSVM&$10^{-2}$&$2^{-12}$&-&-&-&0.19(0.01)&1000(0)&75.99(0.10)\\
2175)&W-LSSVM&$10^{-2}$&$2^{-12}$&-&-&-&0.20(0.01)&1000(0)&76.04(0.10)\\
$l=60$&FS-LSSVM&$10^{-2}$&$2^{-12}$&-&-&-&0.42(0.02)&\textbf{100}(0)&76.66(0.07)\\
&R-LSSVM-W&$10^{-2}$&$2^{-12}$&0.9&-&33.8(7.2)&0.29(0.03)&1000(0)&75.15(0.14)\\
&R-LSSVM-Y&$10^{-2}$&$2^{-12}$&0.9&0.5&\textbf{15.1(2.9)}&0.22(0.01)&947.6(19.5)&80.50(0.04)\\
&SR-LSSVM&$10^{-2}$&$2^{-12}$&0.9&-&25.8(7.1)&0.19($0.01$)&\textbf{100}($0$)&\textbf{81.27(0.03)}\\
\hline
Mushrooms&C-SVC&$10^{0}$&$2^{-3}$&-&-&-&2.33(0.04)&2244.8(27.1)&99.99($0.000$)\\
(5614,&LSSVM&$10^{-1}$&$2^{-3}$&-&-&-&5.75(0.08)&5415.8(0.4)&98.67(0.002)\\
 2708)&W-LSSVM&$10^{-1}$&$2^{-3}$&-&-&-&7.41(0.19)&4837.8(12.2)&99.90(0.001)\\
$l=112$&FS-LSSVM&$10^{-1}$&$2^{-3}$&-&-&-&2.69(0.02)&\textbf{268.9}(0.7)&99.66(0.003)\\
&R-LSSVM-W&$10^{-1}$&$2^{-3}$&0.6&-&26.4(4.0)&7.57(0.22)&5381.8(6.6)&99.97(0.000)\\
&R-LSSVM-Y&$10^{-1}$&$2^{-3}$&0.6&0.3&22.7(4.2)&7.48(0.18)&4928(16.9)&99.71(0.001)\\
&SR-LSSVM&$10^{-1}$&$2^{-3}$&0.6&-&\textbf{6}(0)&\textbf{1.28}($0.01$)&270($0$)&\textbf{100}(0)\\
\hline
\end{tabular}
\begin{tablenotes}
\scriptsize
  \item[1]Pendigits is a pen-based recognition of handwritten digits data set to classify the digits 0 to 9. We only classify the digit 3 versus 4 here.
  \item[2]Protein is a multi-class data set with 3 classes. Here a binary classification problem is trained to separate class 1 from 2.
  \item[3]Satimage is comprised by 6 classes. Here the task of classifying class 1 versus 6 is trained.
  \item[4]USPS is a muti-class data set with 10 classes. Here a binary classification problem is trained to separate class 1 from 2.
\end{tablenotes}
\end{threeparttable}
\label{tab:classification_small_result}
\end{table}

Table \ref{tab:classification_small_result} reports the data information, optimal parameters and experimental results for the medium-scale classification data sets with outliers. The best results are highlighted in bold. In Table \ref{tab:classification_small_result}, we set $r=|B|=0.05 m$ for SR-LSSVM and FS-LSSVM for all data sets except Splice ($r=0.1m$). For C-SVC, the parameter $C=1/(m\lambda)$. All the algorithms independently operate 10 times to get the unbiased results.

As regard to accuracies, Table \ref{tab:classification_small_result} illustrates that our proposed method SR-LSSVM has higher accuracies than any other compared approaches on most data sets. As to training time, our method is faster than other approaches except C-SVC. C-SVC performs well on some medium scale data sets in training speed, but on some larger scale data sets such as Protein and Mushroom, the running speeds of C-SVC are slower than SR-LSSVM. In addition, the accuracies of C-SVC is lower than SR-LSSVM.

In terms of sparseness, SR-LSSVM and FS-LSSVM need much fewer support vectors than other approaches. In other words, these two methods have sparseness. But the accuracy of FS-LSSVM is lower than SR-LSSVM, and FS-LSSVM spends more time than SR-LSSVM on all data sets. C-SVC also displays sparsity but its support vector size is much larger than SR-LSSVM and FS-LSSVM, partly because there exist outliers in the training set.

As respect to iteration times of solving nonconvex programming R-LSSVM, SR-LSSVM needs less iterations than R-LSSVM-W and R-LSSVM-Y to converge to the optimal solution.
\subsubsection{Adult data set experiments}
\begin{figure}[!t]
\centering
\includegraphics[width=1\textwidth]{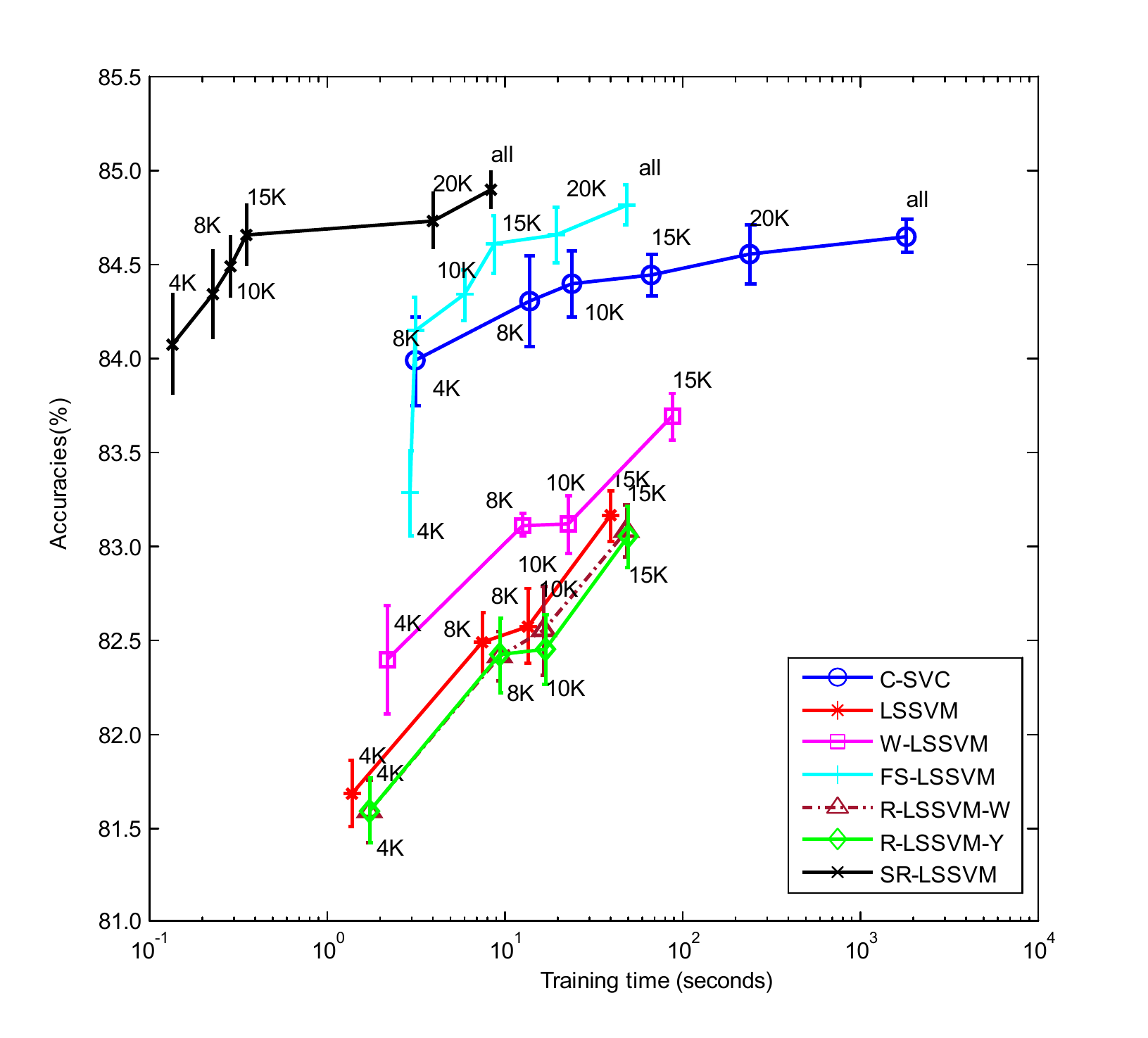}
\caption{The training time, accuracies and their standard deviations of different algorithms on six subsets of the Adult data set with outliers (about 10\%). The horizontal axis is the logarithmic coordinate. The markers 4K to 20K  denote 4000 to 20000 samples in the training stages. 'all' means using all 32561 samples. LSSVM, W-LSSVM, R-LSSVM-W and R-LSSVM-Y are only implemented on the data sets containing less than 15000 samples due to memory limitation of computer. The basic subset size for SR-LSSVM and working set size for FS-LSSVM are both set as 400. In figure, the more left and upper, the better}
\label{fig:adult}
\end{figure}

To investigate the performance of each algorithm on data sets in different sizes, we randomly choose 4000, 8000, 10000, 15000, 20000 and all the 32561 training samples from the training set of Adult data set\cite{lib}. The test set size is 16281.

Fig. \ref{fig:adult} shows the experimental results of all the approaches on the data sets with outliers. The horizontal axis is the logarithmic coordinate in the figure. As to accuracies on these data sets, in general, SR-LSSVM, C-SVC and FS-LSSVM perform better than other methods, and our method SR-LSSVM performs the best. In addition, from Fig. \ref{fig:adult}, we can draw the conclusion that for the medium scale training data sets, especially those with size smaller than 8000, every algorithm runs fast. However, if the training set size exceeds 20000, LSSVM, W-LSSVM, R-LSSVM-W and R-LSSVM-Y cannot operate on our common computer due to lack of memory. So for the large scale benchmark data sets, we do not compare our method SR-LSSVM with LSSVM, W-LSSVM, R-LSSVM-W and R-LSSVM-Y. Moreover, Fig. \ref{fig:adult} also shows that the training time of C-SVC increases rapidly as the sizes of training samples grow larger.

\subsubsection{Large-scale benchmark classification datasets experiments}
Table \ref{tab:classification_big_result} reports the data information, optimal parameters and experimental results for the large-scale data sets with outliers (10\%). We compare our SR-LSSVM with some other sparse algorithms. For Skin-nonskin data set, we randomly select 2/3 of the data as training samples and the rest of the data as testing samples, and for others, we use the default setting in \cite{lib}. All the algorithms operate 5 times independently to get the unbiased results for every dataset. The best results are highlighted in bold.
\begin{table}[!htp]\footnotesize
\centering
\renewcommand\arraystretch{0.85}
\caption{Comparison of the accuracies(\%), training time (seconds) and mean number of support vectors (denoted by nSVs) of SR-LSSVM, PCP-LSSVM, C-SVC, FS-LSSCM and CSI on four large-scale benchmark classification data sets with outliers (10\%). The standard deviations are given in brackets. For C-SVC, the parameter $C=1/(m\lambda)$.}
\begin{threeparttable}
\setlength{\tabcolsep}{2.0pt}
\begin{tabular}{llrrrrrr}
\hline
Data Sets&Algorithms&$m\lambda$&$\sigma$&$\tau$&Training&nSVs&Accuracies(\%)\\
(train, test)&&&&&Time(s)\\
\hline
Skin-nonskin&C-SVC&$10^{0}$&$2^{-2}$&-&1329.4($10.6$)&59910($146.7$)&99.30($0.001$)\\
(163371, 81686)&FS-LSSVM&$10^{-3}$&$2^{-10}$&-&42.6(0.7)&199.2$^1$(0.8)&99.82(0.000)\\
$l=3$&CSI&$10^{-3}$&$2^{-10}$&-&42.9(0.7)&100$^2$(0)&99.84(0.000)\\
&PCP-LSSVM&$10^{-3}$&$2^{-10}$&-&\textbf{32.7}(0.6)&400(0)&\textbf{99.86}(0.000)\\
&SR-LSSVM&$10^{-3}$&$2^{-10}$&1.5&34.8($2.5$)&400($0$)&\textbf{99.86}(0.000)\\
\hline
IJCNN1&C-SVC&$10^{0}$&$2^{-4}$&-&84.3($1.6$)&17103($75.6$)&92.16($0.000$)\\
(49990,91701)&FS-LSSVM&$10^{-3}$&$2^{-8}$&-&34.9(0.8)&398.8(1.0)&94.17(0.014)\\
$l=22$&CSI&$10^{-3}$&$2^{-8}$&-&63.0(1.2)&400(0)&95.10(0.002)\\
&PCP-LSSVM&$10^{-3}$&$2^{-8}$&-&\textbf{2.8}(0.1)&215.6(1.6)&95.39(0.002)\\
&SR-LSSVM&$10^{-3}$&$2^{-8}$&2.5&2.8($0.1$)&215($0$)&\textbf{95.47}($0.002$)\\
\hline
Cod-RNA&C-SVC&$10^{-2}$&$2^{-6}$&-&86.3($0.8$)&22095($42.6$)&94.91(0.000)\\
(59535,271617)&FS-LSSVM&$10^{-4}$&$2^{-1}$&-&36.0(0.8)&399.7(0.5)&94.84(0.001)\\
$l=8$&CSI&$10^{-4}$&$2^{-1}$&-&78.6(0.4)&400(0)&94.99(0.001)\\
&PCP-LSSVM&$10^{-4}$&$2^{-1}$&-&\textbf{6.0}(0.0)&400(0)&94.78(0.001)\\
&SR-LSSVM&$10^{-4}$&$2^{-1}$&1.8&6.4(0.1)&400(0)&\textbf{95.04}(0.001)\\
\hline
Acoustic$^3$&C-SVC&$10^{-1}$&$2^{-4}$&-&278.3($5.6$)&28302($107.0$)&74.87($0.001$)\\
(60562, 15125)&FS-LSSVM&$10^{-6}$&$2^{-4}$&-&75.1(0.5)&399.2(1.0)&\textbf{78.17}(0.002)\\
$l=50$&CSI&$10^{-6}$&$2^{-4}$&-&74.3(0.2)&400(0)&77.62(0.003)\\
&PCP-LSSVM&$10^{-6}$&$2^{-4}$&-&\textbf{40.1}(0.1)&400(0)&77.52(0.001)\\
&SR-LSSVM&$10^{-6}$&$2^{-4}$&1.7&41.4(0.1)&400(0)&77.73(0.001)\\
\hline
\end{tabular}
\begin{tablenotes}
\footnotesize
  \item[1]The working set size is set to 200 for FS-LSSVM, because if we set the working set size at 400, the software cannot operate normally.
  \item[2]The maximal rank of CSI algorithm is 100 for Skin-nonskin data set. If we set this number bigger, the procedure may make an error.
  \item[3]Acoustic data set has three classifications labeled as 1,2,3 respectively. We only choose the second and the third classifications to operate.
\end{tablenotes}
\end{threeparttable}
\label{tab:classification_big_result}
\end{table}

From Table \ref{tab:classification_big_result}, it is obvious that SR-LSSVM also reveals good performances. It has higher prediction accuracies than other algorithms on all data sets except Acoustic. For Acoustic, the accuracy of the SR-LSSVM is a little lower than that of FS-LSSVM method, but the run time of FS-LSSVM is longer than that of SR-LSSVM. Furthermore, it is obvious that C-SVC is the slowest algorithm among these algorithms and the size of its support vectors is much larger than other algorithms. PCP-LSSVM is the fastest algorithm among the compared approaches, but its accuracy is lower than SR-LSSVM.
Based on the above analyses, we conclude that our proposed algorithm SR-LSSVM is more suitable for large-scale classification problems, as it has higher test accuracies, requires less time and produces sparser solutions than other methods.
\subsubsection{Robustness Comparisons for large-scale data set}
In order to compare the robustness of the sparse algorithms PCP-LSSVM, FS-LSSVM, CSI and SR-LSSVM for large-scale data set, we set the rates of outliers at 0\%, 5\% and 10\% respectively on Cod-RNA data set. Table \ref{tab:diffnoise} illustrates the results. From Table \ref{tab:diffnoise}, it is easy to observe that SR-LSSVM has higher accuracies than other approaches, so SR-LSSVM is more robust to outliers than other compared algorithms. Moreover, it is faster than FS-LSSVM and CSI algorithms on data sets with different rates of outliers.
\begin{table}[!htp]\footnotesize
\centering\renewcommand\arraystretch{0.9}
\caption{Comparison the results of Cod-RNA data set with different rates of outliers (0\%, 5\% and 10\%) on four sparse LSSVM based algorithms, which are FS-LSSVM, CSI, PCP-LSSVM and SR-LSSVM. The standard deviations are given in brackets.}
\begin{threeparttable}
\begin{tabular}{l|rrr|c}
\hline
&\multicolumn{3}{c|}{Testing Accuracies(\%)}&{Training Time}\\
\cline{2-4}
{Outlier Rates}&0\%&5\%&10\%&(Seconds)\\
 \hline
{FS-LSSVM}&{96.18(0.001)}&{95.91(0.001)}&{95.03(0.001)}&19.9(0.20)\\
{CSI}&{96.15}(0.000)&{96.02(0.001)}&{94.89(0.002)}&29.0(9.88)\\
{PCP-LSSVM}&{96.31(0.000)}&{96.13(0.000)}&{94.92(0.001)}&\textbf{6.1}(0.15)\\
{SR-LSSVM}&{\textbf{96.32}(0.000)}&{\textbf{96.29}(0.000)}&{\textbf{95.10}(0.001)}&{6.5(0.11)}\\
\hline
\end{tabular}
\end{threeparttable}
\label{tab:diffnoise}
\end{table}

\subsection{Regression experiments}
In regression experiments, we conduct the experiments on four medium scale and four large scale benchmark regression data sets. The data sets Mg, Abalone, and Cadata are downloaded from LIBSVM\cite{lib}, data sets Winequality (WQ), Tic, Relation N D, Slice and 3D-spatial are downloaded from UCI database\cite{Lichman2013}.
Each attribute of the samples is normalized into $[-1,~1]$. For every data set, we randomly select 2/3 of the samples as training set and the rest of the samples as testing set. To test the insensitiveness of our proposed algorithm to outliers, we randomly select 1/10 training samples, and add Gaussian noise $\nu_i \sim N(0,~d^2)$ on their targets to simulate outliers. We set $d$ as half of the mean of the targets for each data set. We set the smoothing parameter $p=10^{4}$ in SR-LSSVM and the stop criterion $\varepsilon=10^{-2}$ , and in $\epsilon$-SVR, $\epsilon=0.01$, $C=1/(m\lambda)$.

\subsubsection{Medium-scale benchmark regression datasets experiments}
\begin{table}[!t]\scriptsize
\centering\renewcommand\arraystretch{0.85}
\caption{Comparison of the RMSE, training time (seconds) and mean number of support vectors (denoted by nSVs) of different algorithms on benchmark regression data sets with outliers. The standard deviations are given in brackets.}
\begin{threeparttable}
\setlength{\tabcolsep}{1.7pt}
\begin{tabular}{llrrrrrrrr}
\hline
Data&Algorithms&$m\lambda$&$\sigma$&$\tau$&$h$&Iterations&Training&nSVs&RMSE\\
(train,&&&&&&&Time(s)\\test)\\
\hline
Mg&$\epsilon$-SVR&$10^{0}$&$2^{-3}$&-&-&-&0.030($0.006$)&863.3($46.8$)&0.129($0.005$)\\
(923,&LSSVM&$10^{-2}$&$2^{1}$&-&-&-&0.108($0.006$)&923($0$)&0.132($0.006$)\\
462)&W-LSSVM&$10^{-2}$&$2^{1}$&-&-&-&0.125($0.003$)&897.8($8.4$)&0.132($0.005$)\\
$l=6$&FS-LSSVM&$10^{-2}$&$2^{1}$&-&-&-&0.041($0.004$)&\textbf{46}($0.9$)&0.126($0.005$)\\
&R-LSSVM-W&$10^{-2}$&$2^{1}$&0.9&-&2(0)&0.109($0.003$)&923($0$)&0.132($0.006$)\\
&R-LSSVM-Y&$10^{-2}$&$2^{1}$&0.9&0.025&2(0)&0.112($0.004$)&923($0$)&0.132($0.006$)\\
&SR-LSSVM&$10^{-2}$&$2^{1}$&0.9&-&2(0)&\textbf{0.023}($0.004$)&\textbf{46}($0$)&\textbf{0.125}($0.004$)\\
\hline
Abalone&$\epsilon$-SVR&$10^{-2}$&$2^{-5}$&-&-&-&0.957($0.095$)&2772.2($2.2$)&2.211($0.071$)\\
(2784,&LSSVM&$10^{-4}$&$2^{-4}$&-&-&-&1.215($0.094$)&2784($0$)&3.271($0.199$)\\
1393)&W-LSSVM&$10^{-4}$&$2^{-4}$&-&-&-&1.468($0.049$)&2784($0$)&2.250($0.217$)\\
$l=8$&FS-LSSVM&$10^{-4}$&$2^{-4}$&-&-&-&0.225($0.004$)&\textbf{48.2}($2.2$)&2.188($0.072$)\\
&R-LSSVM-W&$10^{-4}$&$2^{-4}$&0.01&0.15&20(0)&1.537($0.039$)&2784($0$)&2.251($0.062$)\\
&R-LSSVM-Y&$10^{-4}$&$2^{-4}$&0.01&-&20(0)&1.494($0.039$)&2784($0$)&2.251($0.119$)\\
&SR-LSSVM&$10^{-4}$&$2^{-4}$&0.01&-&20(0)&\textbf{0.060}($0.004$)&48.5($1.7$)&\textbf{2.183}($0.091$)\\
\hline
WQ&$\epsilon$-SVR&$10^{0}$&$2^{-4}$&-&-&-&0.074($0.004$)&911.9($12.4$)&0.652($0.017$)\\
(1066,&LSSVM&$10^{-1}$&$2^{-6}$&-&-&-&0.054($0.003$)&1066($0$)&0.724($0.034$)\\
533)&W-LSSVM&$10^{-1}$&$2^{-6}$&-&-&-&0.080($0.006$)&977.6($5.4$)&0.646($0.022$)\\
$l=11$&FS-LSSVM&$10^{-1}$&$2^{-6}$&-&-&-&0.171($0.006$)&\textbf{42.4}($1.2$)&0.748($0.033$)\\
&R-LSSVM-W&$10^{-1}$&$2^{-6}$&1&-&14.6(3.3)&0.095($0.012$)&1066($0$)&0.647($0.030$)\\
&R-LSSVM-Y&$10^{-1}$&$2^{-6}$&1&0.175&12(1.7)&0.091($0.008$)&944.8($24.4$)&0.646($0.028$)\\
&SR-LSSVM&$10^{-1}$&$2^{-6}$&1&-&12.9(1.9)&\textbf{0.013}($0.001$)&\textbf{42.4}($1.2$)&\textbf{0.646}($0.023$)\\
\hline
Tic&$\epsilon$-SVR&$10^{0}$&$2^{-4}$&-&-&-&2.654($0.083$)&1536.8($38.4$)&0.235($0.005$)\\
(6548,&LSSVM&$10^{0}$&$2^{-6}$&-&-&-&4.294($0.068$)&6540.7($2.5$)&0.229($0.005$)\\
 3274)&W-LSSVM&$10^{0}$&$2^{-6}$&-&-&-&7.196($0.102$)&534.4($24.6$)&0.243($0.006$)\\
$l=85$&FS-LSSVM&$10^{0}$&$2^{-6}$&-&-&-&4.391($0.089$)&\textbf{395.5}($2.1$)&0.235($0.005$)\\
&R-LSSVM-W&$10^{0}$&$2^{-6}$&1&-&2(0)&4.446($0.096$)&6540.7($2.5$)&0.229($0.005$)\\
&R-LSSVM-Y&$10^{0}$&$2^{-6}$&1&0.025&2(0)&4.416($0.043$)&6540.7($2.5$)&0.229($0.005$)\\
&SR-LSSVM&$10^{0}$&$2^{-6}$&1&-&2(0)&\textbf{1.751}($0.022$)&399.9($0.3$)&\textbf{0.229}($0.002$)\\
\hline
\end{tabular}
\begin{tablenotes}
\footnotesize
  \item[*] WQ is the abbreviation of Winequality data set.
\end{tablenotes}
\end{threeparttable}
\label{tab:regression_small_result}
\end{table}

Table \ref{tab:regression_small_result} reports the data information, the optimal parameters and experimental results for medium scale data sets. We set $r=0.05m$ for algorithms SR-LSSVM and FS-LSSVM for every data set except Tic ($r=400$). All the algorithms independently operate 10 times to get the unbiased results.

From Table \ref{tab:regression_small_result}, in terms of prediction errors, it is clear that the proposed approach SR-LSSVM works better than any other approaches for all data sets. Moreover, we can also observe that SR-LSSVM is the fastest method and the standard deviations of running time are the smallest in most cases. The sizes of support vectors of SR-LSSVM and FS-LSSVM are much less than other methods and their standard deviations are quite small, which mean that these two methods are sparseness and the sizes of support vectors are stable. However, FS-LSSVM has worse predicted performance (RMSE) than our method and is more time-consuming than ours. Overall, the predicted errors of our SR-LSSVM algorithm are smaller than other algorithms, and our method greatly saves the training time because of its sparseness. So it is a good choice for regression problems.
\subsubsection{Large-scale benchmark regression datasets experiments}
Table \ref{tab:regression_big_result} reports the data information, the optimal parameters and experimental results for large-scale regression data sets. All the algorithms operate 5 times independently to get the unbiased results for each dataset.
\begin{table}[!t]\footnotesize
\centering
\caption{Comparison of the RMSE, training time (seconds) and mean number of support vectors (denoted by nSVs) of $\epsilon$-SVR, FS-LSSVM, CSI, PCP-LSSVM and SR-LSSVM on four large scale benchmark regression data sets with outliers (10\%). The standard deviations are given in brackets.}
\setlength{\tabcolsep}{2.0pt}
\begin{tabular}{llrrrrrr}
\hline
Data sets&Algorithms&$m\lambda$&$\sigma$&$\tau$&Training time(s)&nSVs&RMSE\\
(train, test)\\
\hline
Cadata&$\epsilon$-SVR&$10^{0}$&$2^{-1}$&-&7.44($0.12$)&8462.8($58.4$)&0.28($0.007$)\\
(13760,6880)&FS-LSSVM&$10^{-3}$&$2^{-1}$&-&1.61(0.03)&300(0.9)&0.24(0.004)\\
$l=8$&CSI&$10^{-3}$&$2^{-1}$&-&6.03(1.59)&\textbf{159}(36.9)&0.25(0.006)\\
&PCP-LSSVM&$10^{-3}$&$2^{-1}$&-&\textbf{1.02}(0.01)&302(0)&0.24(0.005)\\
&SR-LSSVM&$10^{-3}$&$2^{-1}$&1.5&1.05($0.01$)&302(0)&\textbf{0.24}($0.002$)\\
\hline
Relation N D&$\epsilon$-SVR&$10^{-1}$&$2^{-2}$&-&112.6($2.1$)&24556($82.4$)&1.31($0.076$)\\
(35608,17805)&FS-LSSVM&$10^{-2}$&$2^{-4}$&-&2.66(0.03)&173.8(1.9)&1.55(0.060)\\
$l=24$&CSI&$10^{-2}$&$2^{-4}$&-&9.44(1.35)&\textbf{107.4}(21.9)&1.78(0.086)\\
&PCP-LSSVM&$10^{-2}$&$2^{-4}$&-&\textbf{1.58}(0.01)&178(0)&1.41(0.058)\\
&SR-LSSVM&$10^{-2}$&$2^{-4}$&13.7&1.76($0.09$)&178($0$)&\textbf{1.27}($0.028$)\\
\hline
Slice&$\epsilon$-SVR&$10^{-2}$&$2^{-6}$&-&5088.77($23.2$)&35627.4($3.7$)&17.44($0.10$)\\
(35666,17834)&FS-LSSVM&$10^{0}$&$2^{-11}$&-&65.78(0.23)&594.8(0.84)&23.80(0.771)\\
$l=386$&CSI&$10^{0}$&$2^{-11}$&-&161.39(56.5)&\textbf{52}(22.6)&12.92(0.364)\\
&PCP-LSSVM&$10^{0}$&$2^{-11}$&-&\textbf{42.96}(0.25)&600(0)&12.81(0.369)\\
&SR-LSSVM&$10^{0}$&$2^{-11}$&40&43.08($0.27$)&600($0$)&\textbf{8.84($0.054$)}\\
\hline
3D-spatial&$\epsilon$-SVR&$10^{0}$&$2^{-1}$&-&6576.82($46.0$)&276625($234.8$)&0.49($0.006$)\\
(289916,&FS-LSSVM&$10^{-3}$&$2^{-3}$&-&62.51(0.45)&399.4(0.5)&0.45(0.001)\\
144958)&CSI&$10^{-3}$&$2^{-3}$&-&132.3(13.00)&\textbf{173.3}(41.7)&0.46(0.001)\\
$l=3$&PCP-LSSVM&$10^{-3}$&$2^{-3}$&-&\textbf{36.41}(0.13)&400(0)&0.45(0.001)\\
&SR-LSSVM&$10^{-3}$&$2^{-3}$&1.6&37.09($0.35$)&400($0$)&\textbf{0.45}($0.001$)\\
\hline
\end{tabular}
\label{tab:regression_big_result}
\end{table}

From Table \ref{tab:regression_big_result}, 
it is clear that prediction errors of our approach are smaller than those of the other compared algorithms for all data sets. As to the running time, SR-LSSVM method is faster than CSI, FS-LSSVM and $\epsilon$-SVR on all of data sets, and $\epsilon$-SVR is the slowest. For example, for 3D-spatial data set, $\epsilon$-SVR spends about 6577 seconds for training, whereas our method SR-LSSVM only needs approximately 34 seconds and obtains smaller value of RMSE. PCP-LSSVM is the fastest algorithm, but its prediction error is larger than SR-LSSVM. As regard to the number of support vectors, $\epsilon$-SVR has the most, whereas the support vector size of SR-LSSVM is quite small. Although support vector size of CSI is the smallest, the prediction error of it is larger than that of SR-LSSVM, and CSI is slower than our method.
Overall in Table \ref{tab:regression_big_result}, it observes that our proposed approach is more suitable for training large scale regression data sets, as it is quite fast and the prediction performance is good.
\section{Conclusion}
R-LSSVM model is robust for classification and regression problems, which is interpreted in this paper from a re-weighted viewpoint. However, the main shortcoming of it is that the existing algorithms for R-LSSVM lose sparseness, so they cannot deal with the large-scale problems. To overcome this drawback, we utilize the representer theory and incomplete pivoting Cholesky factorization technique to obtain the sparse solution of R-LSSVM in this article, and propose an effective algorithm SR-LSSVM. Experimental results indicate that our algorithm not only has sparseness and robustness, but also has better or comparable prediction performance than the comparison algorithms. Furthermore, the training speed of the proposed algorithm is faster than R-LSSVM-W, R-LSSVM-Y, LSSVM, W-LSSVM, FS-LSSVM, CSI and SVMs (C-SVC for classification and $\epsilon$-SVR for regression) on most median-scale datasets, and for the very large data sets, our method is faster than all the comparison sparse algorithms. So SR-LSSVM is a good choice to deal with the large-scale classification and regression problems.
\section*{Acknowledgements}
This work was supported by the National Natural Science Foundation of China (NNSFC) [grant number 71301067]; and the Fundamental Research Funds for the Central Universities [grant number JB150718].
\section*{References}
\bibliographystyle{elsarticle-num}
\bibliography{References_RLS-SVM}

\begin{thebibliography}{10}
\expandafter\ifx\csname url\endcsname\relax
  \def\url#1{\texttt{#1}}\fi
\expandafter\ifx\csname urlprefix\endcsname\relax\def\urlprefix{URL }\fi
\expandafter\ifx\csname href\endcsname\relax
  \def\href#1#2{#2} \def\path#1{#1}\fi

\bibitem{Suykens1999}
J.~A.~K. Suykens, J.~Vandewalle, {L}east {S}quares {S}upport {V}ector {M}achine
  classifiers, Neural Process. Lett. 9~(3) (1999) 293--300.

\bibitem{Duygu2011}
D.~Calisir, E.~Dogantekin, A new intelligent hepatitis diagnosis system:
  {PCA-LSSVM}, Expert Systems with Applications 38~(8) (2011) 10705--10708.

\bibitem{Long2014}
B.~Long, W.~Xian, M.~Li, H.~Wang, Improved diagnostics for the incipient faults
  in analog circuits using {LSSVM} based on {PSO} algorithm with mahalanobis
  distance, Neurocomputing 133~(10) (2014) 237--248.

\bibitem{Yang2015}
L.~Yang, S.~Yang, S.~Li, R.~Zhang, F.~Liu, L.~Jiao, Coupled compressed sensing
  inspired sparse spatial-spectral {LSSVM} for hyperspectral image
  classification, Knowledge-Based Systems 79 (2015) 80--89.

\bibitem{Mehrkanoon2015}
S.~Mehrkanoon, J.~A. Suykens, Learning solutions to partial differential
  equations using {LS-SVM}, Neurocomputing 159~(2) (2015) 105--116.

\bibitem{Gao2016}
Y.~Gao, X.~Shan, Z.~Hu, D.~Wang, Y.~Li, X.~Tian, Extended compressed tracking
  via random projection based on {MSER}s and online {LS-SVM} learning, Pattern
  Recognition 59 (2016) 245--254.

\bibitem{Suykens2002}
J.~A.~K. Suykens, T.~V. Gestel, J.~D. Brabanter, B.~D. Moor, J.~Vandewalle,
  {L}east {S}quares {S}upport {V}ector {M}achines, Singapore: World Scientific,
  2002.

\bibitem{Valyon2003}
J.~Valyon, G.~Horv¨¢th, A weighted generalized {LS-SVM}, Periodica Polytechnica
  Series Electrical Engineering 47~(3-4) (2003) 229--251.

\bibitem{You2011}
L.~You, L.~Jizhen, Q.~Yaxin, A new {R}obust {L}east {S}quares {S}upport
  {V}ector {M}achine for regression with outliers, Advanced in Control
  Engineering and Information Science 15 (2011) 1355--1360.

\bibitem{KuainiWang2014}
K.~Wang, P.~Zhong, Robust non-convex {L}east {S}quares loss function for
  regression with outliers, Knowledge-Based Systems 71 (2014) 290--302.

\bibitem{XiaoweiYang2014}
X.~Yang, L.~Tan, L.~He, A robust {L}east {S}quares {S}upport {V}ector
  {M}achines for regression and classification with noise, Neurocomputing 140
  (2014) 41--52.

\bibitem{Suykens2000}
J.~A.~K. Suykens, L.~Lukas, J.~Vandewalle, Sparse approximation using {L}east
  {S}quares {S}upport {V}ector {M}achines, in: IEEE International Symposium on
  Circuits and Systems, Geneva, Switzerland, 2000, pp. 757--760.

\bibitem{J.A.K.Suykens2002}
J.~Suykens, J.~D. Brabanter, L.~Lukas, J.~Vandewalle, Weighted {L}east
  {S}quares {S}upport {V}ector {M}achines: robustness and sparse approximation,
  Neurocomputing 48 (2002) 85--105.

\bibitem{Jiao2007}
L.~Jiao, L.~Bo, L.~Wang, Fast sparse approximation for {L}east {S}quares
  {S}upport {V}ector {M}achines, IEEE Transactions on Neural Networks 18~(3)
  (2007) 685--697.

\bibitem{sszhou2016}
S.~Zhou, Sparse {LSSVM} in primal using {C}holesky factorization for large
  scale problems, IEEE Transactions on Networks and Learning Systems 27~(4)
  (2016) 783--795.

\bibitem{Feng2016}
Y.~Feng, Y.~Yang, X.~Huang, S.~Mehrkanoon, J.~A.~K. Suykens, Robust {s}upport
  {v}ector {m}achines for classification with nonconvex and smooth losses,
  Neural Computation 28 (2016) 1217--1247.

\bibitem{Geman1995}
D.~Geman, C.~Yang, Nonlinear image recovery with half-quadratic regularizion,
  IEEE Transactions on Image Processing 4~(7) (1995) 932--946.

\bibitem{Nikolova2005}
M.~Nikolova, M.~K. NG, Analysis of half-quadratic minimization methods for
  signal and image recovery, SIAM Journal on Scientific Computing 27~(3) (2005)
  937--966.

\bibitem{Brabanter2009}
K.~D. Brabanter, K.~Pelckmans, J.~D. Brabanter, M.~Debruyne, J.~Suykens,
  M.~Hubert, B.~DeMoor, Robustness of kernel based regression: a comparison of
  iterative weighting schemes, in: Proceeding of the 19th International
  Conference on Artificial Neural Networks(ICANN), Limassol, Cyprus, 2009, pp.
  100--110.

\bibitem{He2012}
B.~He, X.~Yuan, On the {O}(1/n) convergence rate of the douglas-rachford
  alternating direction method, SIAM Journal on Numerical Analysis 50~(2)
  (2012) 700--709.

\bibitem{Scholkopf2001}
B.~Sch\"{o}lkopf, R.~Herbrich, A.~J. Smola, A generalized representer theorem,
  in: in Proceedings of the Annual Conference on Computational Learning Theory,
  Springer, Amsterdam, Netherlands, 2001, pp. 416--426.

\bibitem{Shai2014}
S.~Shalev-Shwartz, S.~Ben-David, Understanding machine learning-from theory to
  algorithms, Cambridge, 2014.

\bibitem{Yuille2003}
A.~Yuille, A.~Rangarajia, The {c}oncave-{c}onvex {p}rocedure, Neural
  Computation 15~(4) (2003) 915--936.

\bibitem{ShuishengZhou2013}
S.~Zhou, J.~Cui, F.~Ye, H.~Liu, Q.~Zhu, New smoothing {SVM} algorithm with
  tight error bound and efficient reduced techniques, Computational
  Optimization and Applications 56~(3) (2013) 599--617.

\bibitem{Williams2001}
C.~Williams, M.~Seeger, Using the {N}ystr\"{o}m method to speed up kernel
  machines, in: Advances in Neural Information Processing Systems 13, MIT
  Press, 2001, pp. 682--688.

\bibitem{Petros2005}
P.~Drineas, M.~W. Mahoney, On the {N}ystr\"{o}m method for approximating a
  {G}ram matrix for improved kernel-based learning, in: Journal of Machine
  Learning Research, 2005, pp. 2153--2175.

\bibitem{Zhang2010}
K.~Zhang, J.~T. Kwok, Clustered {N}ystr\"{o}m method for large scale manifold
  learning and dimension reduction, IEEE Transactions on Neural Networks
  21~(10) (2010) 1576--1587.

\bibitem{Si2016}
S.~Si, C.~J. Hsieh, I.~S. Dhillon, Computationally efficient {N}ystr\"{o}m
  approximation using fast transforms, in: ICML, 2016.

\bibitem{Tao2014}
T.~P. Dinh, H.~M. Le, H.~A.~L. Thi, F.~Lauer, A difference of convex functions
  algorithm for switched linear regression, IEEE Transactions on Automatic
  Control 59~(8) (2014) 2277--2282.

\bibitem{Bharath2012}
B.~K. Sriperumbudur, G.~R.~G. Lanckriet, A proof of convergence of the{
  C}oncave-{C}onvex {P}rocedure using {Z}angwill¡¯s theory, Neural Computation
  24~(6) (2012) 1391--1407.

\bibitem{PHAM1997}
P.~D. TAO, L.~T.~H. AN, Convex analysis approach to {D. C. }programming:
  Theory, algorithms and applications, Acta Mathematic Vietnamica 22~(1) (1997)
  289--355.

\bibitem{Brabanter2011_lssvmtoolbox}
K.~D. Brabanter, P.~Karsmakers, F.~Ojeda, C.~Alzate, J.~D. Brabanter,
  K.~Pelckmans, B.~D. Moor, J.~Vandewalle, J.~Suykens,
  \href{http://www.esat.kuleuven.be/sista/lssvmlab/}{Lssvmlab v1.8} (Augest
  2011).
\newline\urlprefix\url{http://www.esat.kuleuven.be/sista/lssvmlab/}

\bibitem{Bach_csi2005}
F.~R. Bach, M.~I. Jordan, Predictive low-rank decomposition for kernel methods,
  in: Proceedings of the 22 nd International Conference on Machine Learning,
  Bonn, Germany, 2005, pp. 33--40.

\bibitem{lib}
C.-C. Chang, C.-J. Lin,
  \href{http://www.csie.ntu.edu.tw/~cjlin/libsvm}{{LIBSVM} : a library for
  {S}upport {V}ector {M}achines.} (2011).
\newline\urlprefix\url{http://www.csie.ntu.edu.tw/~cjlin/libsvm}

\bibitem{Lichman2013}
M.~Lichman, \href{http://archive.ics.uci.edu/ml}{{UCI} {M}achine {L}earning
  {R}epository} (2013).
\newline\urlprefix\url{http://archive.ics.uci.edu/ml}

\end{thebibliography}

\end{document}